%% file: arxiv_main.tex
\documentclass{article}

\usepackage{arxiv}

\usepackage[utf8]{inputenc} %
\usepackage[T1]{fontenc}    %
\usepackage{hyperref}       %
\usepackage{url}            %
\usepackage{booktabs}       %
\usepackage{amsfonts}       %
\usepackage{nicefrac}       %
\usepackage{microtype}      %
\usepackage{lipsum}		%
\usepackage{graphicx}
\usepackage{natbib}
\usepackage{doi}

\usepackage{hyperref}
\usepackage{url}

\usepackage[utf8]{inputenc}
\usepackage[T1]{fontenc}
\usepackage{microtype}
\usepackage{amsmath}
\usepackage{amssymb}
\usepackage{amsthm}
\usepackage{graphicx}
\usepackage{booktabs}
\usepackage{algorithm}
\usepackage{algorithmic}
\usepackage{subcaption}
\usepackage{xcolor}
\usepackage{cleveref}

\graphicspath{{figures/}{./figures/}{./}}

\usepackage{xspace}
\usepackage{wrapfig}
\input{math.tex}

\newcommand{\method}{\textsc{Metro}-WM\xspace}
\newcommand{\methodfull}{\textsc{Metro}-WM\xspace} %
\newcommand{\methodlewm}{\textsc{Metro}-LeWM\xspace}
\newcommand{\methodpldm}{\textsc{Metro}-PLDM\xspace}

\newcommand{\zobs}{z_{\mathrm{obs}}}
\newcommand{\zgoal}{z_{\mathrm{goal}}}
\newcommand{\zsub}{z_{\mathrm{sub}}}
\newcommand{\pilo}{\pi_{\mathrm{lo}}}
\newcommand{\pihi}{\pi_{\mathrm{hi}}}

\newcommand{\Dset}{\mathcal{D}}
\newcommand{\Real}{\mathbb{R}}
\newcommand{\keypoint}[1]{\noindent\textbf{#1}\quad}
\theoremstyle{plain}
\newtheorem{proposition}{Proposition}
\theoremstyle{definition}

\newcommand{\change}[1]{%
    #1%
}

\title{\methodfull: Long-Horizon Latent Planning\\
       with Realisable Sub-Goals}

\author{ 
    Royson Lee \\
    Samsung AI\\
    Cambridge, UK\\
	\texttt{royson.lee@samsung.com} \\
	\And
	Fady Rezk \\
    Samsung AI, Cambridge, UK\\
    University of Edinburgh, UK\\
    \texttt{f.rezk@samsung.com} \\
	\And
	Titouan Parcollet \\
	Samsung AI\\
    Cambridge, UK\\
	\texttt{t.parcollet@samsung.com} \\
	\AND
	Timothy Hospedales \\
	Samsung AI, Cambridge, UK\\
    University of Edinburgh, UK\\
	\texttt{t.hospedales@samsung.com} \\
    \And
	Cristina Cornelio \\
	Samsung AI\\
    Cambridge, UK\\
	\texttt{c.cornelio@samsung.com} \\
}

\date{}

\renewcommand{\shorttitle}{\methodfull: Long-Horizon Latent Planning with Realisable Sub-Goals}

\begin{document}
\maketitle

\vspace{-3em}
{\centering Code: \url{https://github.com/SamsungLabs/metro-wm}\par}
\vspace{1.5em}

\begin{abstract}
Model-predictive control with Joint-Embedding Predictive Architectures (JEPAs) provides a strong zero-shot goal-reaching planner, but it is only effective over short planning horizons. Hierarchical extensions attempt to bridge this gap by learning a macro planner to predict intermediate latent sub-goals to guide the micro planner. In this work, we demonstrate that unconstrained latent sub-goal prediction is fundamentally flawed. A rigorous evaluation reveals that \change{leading hierarchical macro planners routinely emit} physically unrealisable sub-goals. To resolve this, we introduce \method, a hierarchical framework that issues sub-goals by retrieving genuine states from prior experience rather than generating ungrounded latent vectors. Specifically, \method constructs a graph whose vertices are observed frames from offline expert demonstrations or random-action trajectories, allowing frames from different episodes to be connected and stitched into routes to the goal. Planning over the full graph also makes the system highly robust to execution errors: if the micro planner drifts off course, \method instantly finds a new optimal path from the current state. Our experiments show that \method achieves superior long-horizon success rates of up to $37.33$ percentage points over the \change{strongest} hierarchical \change{baseline on the same world model} while being up to $10.9\times$ faster, requiring both $13\text{--}56\times$ less offline compute and fewer tuned hyperparameters. Additional analysis reveals that \method{} \change{frequently} finds shorter paths than the offline demonstrations, outperforms an oracle relying on the query's own demonstration \change{on most benchmarks}, and maintains robust performance under extremely sparse dataset conditions.
\end{abstract}

\section{Introduction}
A central goal in autonomous machine intelligence is to learn an internal model of the world that is capable of predicting how the physical environment transforms when an action is applied~\citep{lecun2022path}. However, predicting every low-level pixel or sensory detail of future states is often computationally intractable, and in many cases, irrelevant for decision-making. One of the ways to overcome this is for the world model to learn abstract, high-level dynamics in a reward-free manner, allowing it to be used at test time to plan towards novel goals that were not seen during training.

Joint Embedding Predictive Architectures (JEPAs)~\citep{lecun2022path,assran2023ijepa,bardes2024vjepa} 
provide a concrete realisation of this idea by learning compact representations through prediction in latent space. 
JEPA-based world models~\citep{zhou2024dinowm,sobal2025pldm,maes2026lewm} build on these representations by pairing them with an action-conditioned latent dynamics model, making planning a viable alternative to policy learning.
These models enable zero-shot, goal-conditioned model-predictive control (MPC),
which evaluates candidate action sequences by comparing their predicted latent outcomes with the encoded goal.
However, such flat planners are reliable only over short horizons: as the number of steps increases, prediction errors accumulate and latent distances become less informative. 
Yet many tasks we would like to perform, such as navigating around obstacles or multi-step object manipulation, require long sequences of actions. To be able to fully realise the promise of learned world models, we therefore need planners that remain reliable beyond short, locally solvable tasks.

To extend the planning horizon, recent hierarchical JEPA frameworks\change{~\citep{zhang2026hierarchical, masip2026ffjepa, zhang2026hjepa}} add a macro planner that generates intermediate latent sub-goals, decomposing long-horizon tasks into segments that a micro planner can solve reliably. 
However, these sub-goals are not constrained to the small subset of latent space corresponding to physically realisable states.
\citet{caselli2026mindthegap} found that unconstrained search around macro-actions generate poor quality sub-goals and proposed constraining the search to macro-actions derived from training trajectories.
Extending these observations, we introduce the \emph{realisability residual}, a decoder-free measure of sub-goal realisability that avoids the visual artefacts of  
decoders~\citep{kurutach2018causalinfogan,nair2018rig,pertsch2020lvd}. 
Using this measure, we show that existing mitigations 
remain insufficient: \change{recent} macro planners still routinely generate sub-goals that do not correspond to any physically realisable state.

\begin{figure}[t]
  \centering
  \includegraphics[width=\linewidth]{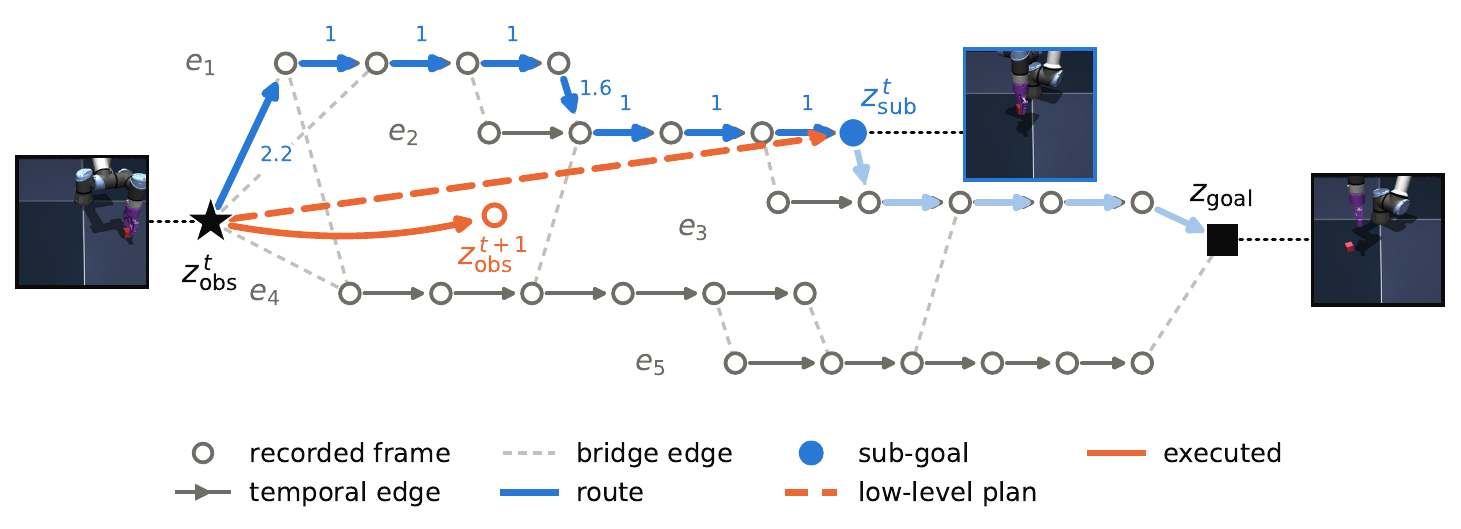}
  \caption{\textbf{One cycle of \method.} Rather than generating sub-goals, \method retrieves them from previously seen episodes.
  The macro planner searches the graph whose vertices represent observed frames encoded as latent states, with $e_1 \dots e_5$ in the figure indicating the specific episodes containing these frames. 
  Directed \textit{temporal edges} (solid grey) connect consecutive frames within each episode, while
  symmetric \textit{bridge edges} (dashed grey) link visually similar frames across episodes (example endpoints in Appendix Figure~\ref{fig:stitch}). 
  All edge costs are expressed in frame units: Temporal edges have unit cost, while bridge-edge costs are converted into frame-equivalent costs derived from empirical dataset quantiles.
  Given a goal $\zgoal$, we use Dijkstra's algorithm to search backward from $\zgoal$ to find the lowest-cost route from every graph vertex. 
  At each planning cycle, \method add the current observation $\zobs^t$ to the graph, and identifies the shortest route (blue) to $\zgoal$ (which may stitch together multiple episodes).
  The furthest frame along this route that is within the micro planning horizon is then selected as the next sub-goal $\zsub^t$. 
  After execution (solid orange), planning resumes from the reached state $\zobs^{t+1}$ (see Appendix Figure~\ref{fig:reentry}).}
  \label{fig:steps}
\end{figure}

This lack of a realisability guarantee motivates our method, \method, a hierarchical framework that replaces
latent sub-goal generation with retrieval from recorded episodes\footnote{We use the terms episodes and demonstrations interchangeably.}. See Figure~\ref{fig:steps} for an overview.
Akin to navigating a transit network, \method's macro planner searches a graph of observed frames for routes to the goal, then selects vertices along the route as sub-goals for the micro planner.
By retrieving sub-goals exclusively from recorded episodes, \method guarantees that each sub-goal encodes a physically realisable state.

Our contributions are threefold: 
\textit{i)} The realisability residual, a decoder-free measure showing that the \change{recent} hierarchical macro planners generate sub-goals that do not correspond to encodings of realisable states; 
\textit{ii)} \method, a non-parametric macro planner that builds a graph over recorded (expert or random-action) episodes, connecting consecutive frames within each episode and visually similar frames across episodes. The planner uses shortest-path search to find routes (potentially stitching together multiple trajectories) and selects graph nodes along them as sub-goals, thereby guaranteeing their physical validity;
\textit{iii)} An evaluation on four benchmarks with two distinct \change{world models} LeWM~\citep{maes2026lewm} and PLDM~\citep{sobal2025pldm} where \method achieves superior long-horizon success rates by up to $37.33$ pp over the next-best hierarchical approach while being up to $10.9\times$ faster and requiring $13-56\times$ less offline compute. \method also frequently finds paths shorter than the expert demonstrations (when provided), outperforms an oracle that uses ground-truth sub-goals, and remains robust in very sparse data regimes. \change{Under the protocol of the concurrent H-JEPA~\citep{zhang2026hjepa}, with its released flat world model and gradient-based planner, \method matches or outperforms two- to four-level H-JEPA and HWM~\citep{zhang2026hierarchical} hierarchies on four benchmarks.}

\section{Related work}
\label{sec:related}

Hierarchical planning tackles the challenges of long-horizon control by decomposing a task into shorter segments, with sub-goals either \emph{generated} by a trained model or \emph{retrieved} from recorded experience. 
Within the domain of JEPAs and latent world models, existing hierarchical approaches are overwhelmingly generative, relying on continuous optimization and/or learned macro-policies to predict latent sub-goals~\citep{zhang2026hierarchical,caselli2026mindthegap,masip2026ffjepa}. \change{Concurrently with our work, H-JEPA~\citep{zhang2026hjepa} trains a stack of JEPA end to end and plans top-down with predicted latent sub-goals.}
Conversely, the paradigm of retrieving sub-goals from data is well-established, but has been predominantly confined to classical reinforcement learning (RL)~\citep{eysenbach2019sorb,emmons2020sgm} and visual navigation~\citep{savinov2018sptm,shah2021ving}, where agents query episodic memory buffers and/or traverse spatial topological maps. 

\method bridges this gap, bringing the physical guarantees of retrieval-based sub-goals to JEPA-based latent planning. This adaptation introduces three fundamental distinctions: \textit{1)} we tackle the distinct problem of guiding a zero-shot latent micro planner while strictly preventing sub-goal hallucination rather than maximizing extrinsic rewards. \textit{2)} we calibrate frame-based edge weights directly from the empirical dataset quantiles instead of relying on trained value functions or critics. \textit{3)} our macro planner remains completely non-parametric and training-free.
An extended discussion on previous works can be found in Appendix Section~\ref{app:related}.

\section{Preliminaries}
\label{sec:prelim}

\paragraph{Goal-conditioned latent planning.}
Let $\phi$ be a frozen encoder mapping an observation $o$ to a latent state $z = \phi(o) \in \Real^{D}$, and $f$ a frozen latent dynamics model (predictor) trained alongside
it, so $f(z_t, a_t)$ predicts the encoding of the next observation under action
$a_t$. We take $(\phi, f)$ as given; in our experiments
we use the released LeWorldModel (LeWM)~\citep{maes2026lewm} and PLDM~\citep{sobal2025pldm} checkpoints\change{, and the flat LeWM checkpoints released with H-JEPA}. Following convention, we chunk actions in blocks of $5$ environment frames so a plan of $h$ blocks spans $5h$ frames, where $h$ is the latent planner's planning horizon. Given a goal observation, $o_{\mathrm{goal}}$, and writing
$\zgoal = \phi(o_{\mathrm{goal}})$ and $\zobs = \phi(o_t)$, a flat latent planner solves, at every replanning step,
\begin{equation}
  a^{\star}_{1:h} \;=\; \argmin_{a_{1:h}} \;
    \bigl\lVert \hat{z}_{h} - \zgoal \bigr\rVert_2^{2},
  \qquad
  \hat{z}_{0} = \zobs, \quad
  \hat{z}_{t+1} = f(\hat{z}_{t}, a_{t}),
  \label{eq:flat-mpc}
\end{equation}
executes a prefix of $a^{\star}_{1:h}$, re-encodes, and repeats, minimizing by
a sampling optimiser. Following LeWM, we adopt the cross-entropy method~\citep{rubinstein1999cem} (CEM) in all our experiments\change{, except those of Section~\ref{sec:results-hjepa}, which use H-JEPA's gradient-based planner}.
Eq.~\ref{eq:flat-mpc} evaluates new goals zero-shot by relying solely on a terminal cost. Consequently, it is only effective if the latent distance metric $\lVert \cdot - \zgoal \rVert$ can successfully distinguish between states across the entire length of the planned trajectory. This strictly limits $h$: as $h$ lengthens, $f$ accumulates prediction errors and the latent distance metric saturates. Saturation serves as the fundamental bottleneck: even a perfect predictor cannot rank distant candidates, as demonstrated in Appendix Section \ref{app:saturation}.
We call Eq.~\ref{eq:flat-mpc} the \emph{micro} planner $\pilo$, reliable only at short planning horizons.

\paragraph{Hierarchical latent planning.}
To enable longer horizon planning, hierarchical latent planners insert a \emph{macro} level. A
macro planner $\pihi$ is given $(\zobs, \zgoal)$ and emits a latent sub-goal
\begin{equation}
  \zsub \;=\; \pihi(\zobs, \zgoal),
  \qquad
  \pihi : \Real^{D} \times \Real^{D} \;\to\; \Real^{D},
  \label{eq:hier}
\end{equation}

which replaces $\zgoal$ in Eq.~\ref{eq:flat-mpc} so $\pilo$ only ever covers a distance it is competent over.
While implementations of $\pihi$ vary, ranging from learned macro actions and sub-goal policies to optimised expert actions, they all follow the structure of Eq.~\ref{eq:hier}. However, we identify a core issue with this formulation: $\zsub$ is produced by optimising over the entire latent space $\Real^{D}$, but $\Real^{D}$ is much larger than the set of valid, encodable states. If $\mathcal{S}$ represents the environment states, the set of realisable latents is $\mathcal{Z} = \{\, \phi(\mathrm{render}(s)) : s \in \mathcal{S} \,\}$ where $\mathrm{render}(\cdot)$ maps the underlying physical state to the visual observation. Because $\mathcal{S}$ is often low-dimensional, $\mathcal{Z}$ forms a very thin manifold within $\Real^{D}$. Since nothing forces the optimization to keep $\zsub$ inside $\mathcal{Z}$, the planner can output a sub-goal outside this manifold, representing a physically non-existent state.

\section{Are macro latent sub-goals realisable?}
\label{sec:realizability}

Previous works~\citep{zhang2026hierarchical, caselli2026mindthegap} visualise latent sub-goals using a decoder, trained on demonstrations from the dataset used to train the models,
which confounds sub-goal quality with the decoder's reconstruction ability. 
We, therefore, introduce an evaluation protocol isolating the evaluation to the environment's renderer and frozen encoder $\phi$, while remaining strictly independent of the provided demonstrations. 
Specifically, we define the \textbf{realisability residual} as:
\begin{equation}
  r(z) \;=\; \min_{s \,\in\, \mathcal{S}} \;
    \bigl\lVert\, \phi\bigl(\mathrm{render}(s)\bigr) - z \,\bigr\rVert_2
  \;=\; \mathrm{dist}\bigl(z, \mathcal{Z}\bigr).
  \label{eq:rz}
\end{equation}
where a large $r(z)$ means no physical state encodes to $z$. Exhaustive search over $\mathcal{S}$ is intractable, so we approximate the minimum with CEM. CEM only upper-bounds $r(z)$, so we rule out under-searching in two ways: we tune every CEM hyperparameter towards lower $r$, and we verify that substantially increasing the search budget leaves the median $r(z)$ unchanged (Appendix~\ref{sec:appendix-realizability}).

\begin{figure}[t]
  \centering
  \includegraphics[width=\linewidth]{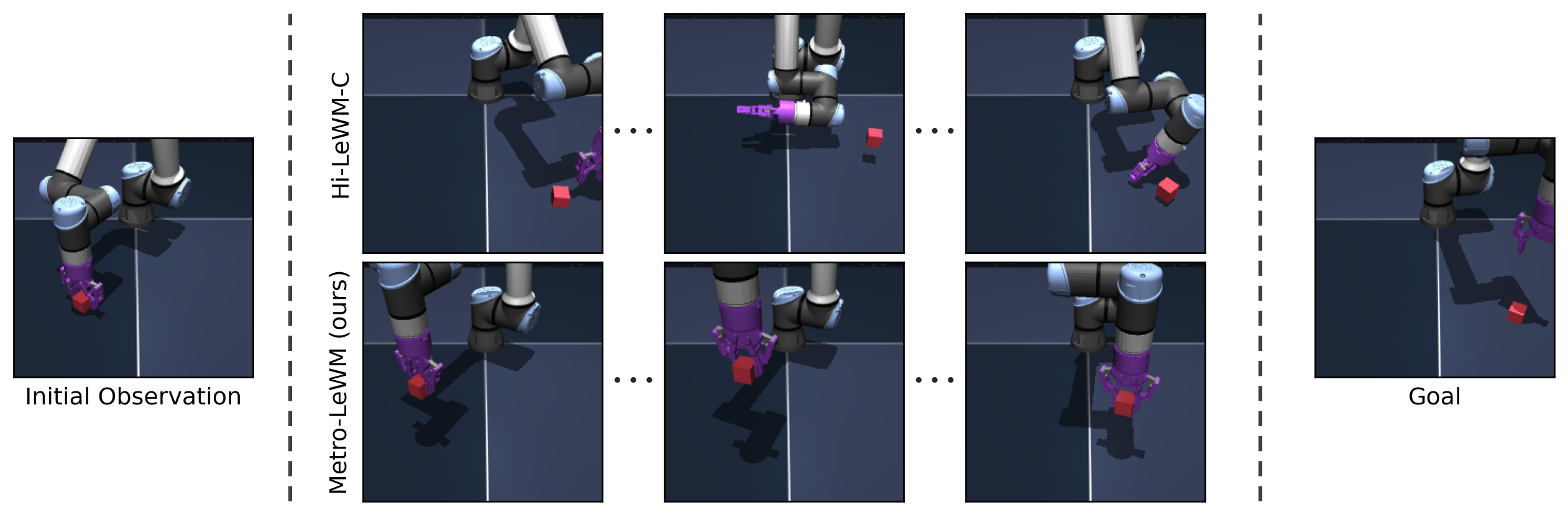}
  \caption{Decoder-free subgoal analysis. Top row: closest \emph{realizable} states to the non-realisable sub-goals emitted by Hi-LeWM-C macro planner (Eq.~\ref{eq:rz}), drawn by the environment's own renderer. Bottom row: Our  \method sub-goals, which are real states retrieved from the training dataset. Micro  actuation - constrained to \emph{realisable} states - will not progress toward the final goal when pursuing sub-goals issued by Hi-LeWM-C, but will progress when using \method sub-goals.}
  \label{fig:realizability-cube}
\end{figure}

We apply this to the macro-planner sub-goals of \change{Hi-LeWM-C~\citep{caselli2026mindthegap}}, on \textsc{OGB-Cube}~\citep{park2025ogbench} and \textsc{PushT}~\citep{chi2023diffusionpolicy}. On both, CEM finds no nearby physical state: the median $r(z)$ is $0.9$ on \textsc{Cube} and $2.3$ on \textsc{PushT} (Fig.~\ref{fig:realizability-cube}; Appendix Fig.~\ref{fig:realizability-pusht})\change{\footnote{\change{We show the same holds for subgoals of H-JEPA and HWM in Appendix~\ref{sec:appendix-realizability}}}}. The scale of $r(z)$ is not interpretable in isolation. However, a residual that stays above zero as the search budget grows indicates that no realisable state exists. This motivates our approach, which emits only sub-goals that are encodings of real states.

\section{\method: sub-goal selection by shortest paths}
\label{sec:method}

Section~\ref{sec:realizability} highlights that latent sub-goals in previous works are often unrealisable, yielding a clear design constraint: the sub-goal issued to $\pilo$ must correspond to a physically plausible state. Any previously recorded frame trivially meets this requirement and the challenge then lies in selecting a \emph{useful} state, one that is both within $\pilo$'s reach and on a route to meet the goal. To achieve this, \method replaces the prediction of Eq.~\ref{eq:hier} with a shortest-path search over prior experience. 

\subsection{The graph of recorded frames}

To enable a shortest-path search over prior experience, we model the dataset as a weighted graph $G = (V, E, w)$.
Assuming a corpus of $M$ recorded trajectories $\Dset = \{\tau^{(m)}\}_{m=1}^{M}$, where
$\tau^{(m)} = (o^{(m)}_{1}, a^{(m)}_{1}, \dots, a^{(m)}_{L_m - 1}, o^{(m)}_{L_m})$ has length $L_m$, the vertices are the encoded frames, $z^{(m)}_{t} = \phi(o^{(m)}_{t})$: $V = \{ (m,t) : 1 \le m \le M,\, 1 \le t \le L_m \}$. The edges $E$ connect these frames and are divided into two types:
\\\textbf{Temporal edges} join consecutive frames within the same episode, $E_{\mathrm{tmp}} = \{ (m,t) \to (m,t{+}1) \}$, each costing one frame ($w(e) = 1$). They are transitions a demonstrator performed, making them feasible and strictly \emph{directed} by construction as many manipulation tasks are irreversible.
\\\textbf{Bridge edges} connect a vertex to its $k$-nearest neighbours (kNN) \emph{cross-episodes}, within a radius $\varepsilon$:
\begin{equation}
  E_{\mathrm{brg}} = \bigl\{ u \leftrightarrow v \;:\;
    v \in \mathrm{kNN}_{k}(u),\;
    m(v) \neq m(u),\;
    \lVert z_u - z_v \rVert_2 \le \varepsilon \bigr\}.
  \label{eq:bridge}
\end{equation}
A bridge asserts that two frames from different demonstrations represent a similar physical situation, so it is symmetric; bridges are what allow a route to leave one demonstration and continue along another. Since $k$-nearest-neighbour retrieval is not itself symmetric, we symmetrize the retrieved block by elementwise maximum, storing each undirected bridge exactly once.

\subsection{Bridge edge weights}
A naive bridge weight based directly on latent distance violates the additive property necessary for shortest-path search. Because latent spaces saturate (Appendix Section~\ref{app:saturation}), accumulating latent distances fails to properly rank long, multi-hop routes. Hence, we express every bridge weight in frames, converting latent distance into frames using the demonstrations themselves.  For each frame separation $d \in {1, \dots, H}$, we compute $\Lambda(d) = Q_{\alpha}(\{ \lVert z^{(m)}_{t+d} - z^{(m)}_{t} \rVert_2 : (m,t) \})$, which represents a low empirical quantile ($\alpha$) of latent distances over that specific time gap. Restricting the grid to the waypoint budget $H$ avoids extrapolating past $\pilo$'s reliable range and keeps $\Lambda$ below saturation. Furthermore, we enforce monotonicity by taking the running maximum of $\Lambda(d)$ wherever the empirical quantile dips, ensuring that a larger latent gap is never charged fewer frames. The bridge weight, $w(u \leftrightarrow v) = \kappa(\lVert z_u - z_v\rVert_2)$, is the piecewise-linear interpolant of the knots ${(\Lambda(d), d)}_{d=1}^{H}$. 
For a latent gap $\delta$ in the interval $[\Lambda(d), \Lambda(d+1)]$, this interpolant is calculated and clipped to range:

\begin{equation}
\kappa(\delta) =
\operatorname{clip}\left(
d + \frac{\delta - \Lambda(d)}{\Lambda(d+1) - \Lambda(d)},
1, H
\right).
\label{eq:kappa}
\end{equation}
A \emph{low} quantile $\alpha$ makes $\kappa$ deliberately conservative. By mimicking a slow demonstrator, it charges more frames for a gap, preventing the dangerous failure mode of under-pricing, which would trick the search into buying un-executable shortcuts. 

\keypoint{Consolidating hyperparameters.}
The bridge radius $\varepsilon$, quantile $\alpha$, and waypoint budget $H$ all redundantly control the maximum sub-goal distance. To simplify tuning, we fix $\alpha$ to a conservative constant ($\alpha = 0.25$) and set $\varepsilon = \Lambda(H)$, leaving $H$ as the \emph{sole hyperparameter}. In our experiments, we tune only $H$ alongside the micro planner's horizon $h$.

\subsection{Graph search and sub-goal selection}
\label{sec:dijkstra}
To initiate the search for each evaluation query, we insert its goal as a vertex $g$ with latent representation $\zgoal$, connecting it to its $k$-nearest admissible vertices via bridge edges. To ensure fair evaluation, \change{if the query is drawn from a recorded episode $m^{\star}$ in the graph,} we delete \change{$m^{\star}$} from the graph. We then run Dijkstra's algorithm~\citep{dijkstra1959} backward from $g$, yielding a cost-to-go $J(u)$ and an optimal successor $\mathrm{succ}(u)$ for every vertex:
\begin{equation}
  J(u) = \min_{\text{paths } u \rightsquigarrow g} \sum_{e} w(e),
  \qquad
  \mathrm{succ}(u) = \argmin_{v \,:\, u \to v} \bigl[ w(u \to v) + J(v) \bigr].
  \label{eq:dijkstra}
\end{equation}
Crucially, because this search, run once per query, computes a global value function rather than a single trajectory, it is highly robust to execution errors. Regardless of where the micro planner ($\pilo$) lands after attempting an action, the current state can query the optimal continuation without needing to recalculate a path (See Appendix Figure~\ref{fig:reentry}). 

\keypoint{Connecting the current state.}
Because the current observed state $\zobs$ is continuously changing and is not a pre-existing vertex, it must be attached to the graph dynamically. Specifically, we evaluate the optimal path from $\zobs$ to the $\zgoal$ by applying a one-step Bellman backup over its $k$-nearest neighbours in the graph. The neighbour that minimizes this backup is the optimal entry point $u^{\star}$:
\begin{equation}
  u^{\star} \;=\; \argmin_{u \,\in\, \mathrm{kNN}_{k}(\zobs),\;
                           \lVert \zobs - z_u \rVert \le \varepsilon}
    \; \Bigl[ \underbrace{\kappa\bigl( \lVert \zobs - z_u \rVert_2 \bigr)}_{\text{cost of entering}}
      \;+\; \underbrace{J(u)}_{\text{cost of finishing}} \Bigr],
  \qquad
  J(\zobs) = \min_{u} \bigl[ \cdot \bigr],
  \label{eq:entry}
\end{equation}
where the entry cost is evaluated as a directed bridge, $w(\zobs \to u) = \kappa(\lVert \zobs - z_u \rVert_2)$. Because $\zobs$ is inserted with out-edges only, this backup yields the exact shortest-path cost $J(\zobs)$ from the current state to the goal (Proof in Appendix Section~\ref{app:proofs}). If no candidate lies within the $\varepsilon$ radius, we aggressively attach to the nearest candidate regardless of distance to ensure the micro planner always has a valid path forward.

\keypoint{Selecting a sub-goal.}
Starting from the entry vertex $u^{\star}$, we traverse the shortest path, accumulating the total weight cost $C(\ell)$. We select the farthest vertex on this path that remains \emph{within} the waypoint budget $H$:
\begin{equation}
  \begin{gathered}
    u_{0} = u^{\star}, \qquad u_{j+1} = \mathrm{succ}(u_j),
    \qquad
    C(\ell) = w(\zobs \to u_0) + \sum_{j=0}^{\ell-1} w(u_j \to u_{j+1}), \\[2pt]
    \ell^{\star} = \max\bigl\{\, \ell \ge 1 \;:\; C(\ell) \le H \,\bigr\},
    \qquad
    \zsub = z_{u_{\ell^{\star}}} .
  \end{gathered}
  \label{eq:walk}
\end{equation}
The condition $\ell \ge 1$ guarantees that the search advances by at least one hop, even if the initial entry cost $w(\zobs \to u_0)$ immediately overshoots the budget, ensuring that the walk always advances. In the case where there is no valid path from $\zobs$ to $\zgoal$, i.e. the graph is disjoint, we route towards the reachable vertex closest to the goal in latent space. Once graph progress is exhausted, we emit $\zgoal$ directly, allowing the system to gracefully degrade to direct goal-seeking rather than failing outright.

\subsection{Summary: end-to-end pipeline}
\label{sec:method-pipeline}

The \method pipeline operates across three distinct stages:

\textbf{Offline preprocessing (once per corpus):} We encode the dataset $\Dset$, estimate the empirical quantile function $\Lambda$, and set the bridge radius $\varepsilon = \Lambda(H)$. We then retrieve the $k$-nearest neighbours for every vertex to build the temporal edges $E_{\mathrm{tmp}}$ and bridge edges $E_{\mathrm{brg}}$, pricing them via $\kappa$ (Eq.~\ref{eq:kappa}).

\textbf{Planning (once per episode):} We encode the target goal observation, insert it into the graph as $\zgoal$, and delete the query's ground-truth episode $m^{\star}$ \change{if the graph contains it}. We then run Dijkstra's algorithm from $\zgoal$ to compute the distance-to-goal $J$ and the optimal successor $\mathrm{succ}$ for the entire graph (Eq.~\ref{eq:dijkstra}).

\textbf{Online (every macro step):} We encode the current observation $\zobs$, enter the graph via Eq.~\ref{eq:entry}, walk the successor function within the waypoint budget $H$, and hand the selected sub-goal $\zsub$ to micro planner $\pilo$ (Eq.~\ref{eq:walk}). Because $\pilo$ is imperfect, it will inevitably land in an unpredicted state. At the next step, we simply re-encode this new state and re-enter the graph. 

We show the graph traversal and sub-goal selection in Fig.~\ref{fig:steps} and re-entry in Appendix Fig.~\ref{fig:reentry}. We also provide the full algorithm in Appendix Section~\ref{app:algorithm}.

\section{Evaluation}
\label{sec:results}

\subsection{Experimental Setup}
\label{sec:results-setup}

\keypoint{Model and Benchmarks.} We use LeWM~\citep{maes2026lewm} and PLDM~\citep{sobal2025pldm} as our frozen \change{world models} in all experiments\change{, except those in Section~\ref{sec:results-hjepa}, which follow the protocol of H-JEPA~\citep{zhang2026hjepa}}. For \method, we plug \change{their CEM micro planners} into \method, referred as \methodlewm and \methodpldm respectively, and use \change{each world model's} shipped checkpoints and CEM~\citep{rubinstein1999cem} configurations and hyperparameters, tuning only its planning horizon $h$. We evaluate on the four goal-conditioned benchmarks LeWM was trained on: \textsc{PushT}~\citep{chi2023diffusionpolicy,florence2021ibc}, OGBench-\textsc{Cube}~\citep{park2025ogbench}, \textsc{TwoRoom}~\citep{zhou2024dinowm,sobal2025pldm}, and \textsc{Reacher}~\citep{tassa2018deepmind}.

\keypoint{Metrics and Baselines.} A goal is the observation $d$ frames after the episode start, with $d \in \{25, 50, 75\}$. Each result is evaluated on the same 50 query episodes per benchmark, repeated thrice on different planner seeds. We report mean and standard deviation of the total success rate and latency at varying solver budgets ($T$) over all episodes. Our baselines include flat \change{planning} with LeWM and PLDM, hierarchical LeWM (Hi-LeWM), Hi-LeWM-C, and its staged variant~\citep{caselli2026mindthegap}. 

\keypoint{Hyperparameters.} For \change{flat planning}, we extend their horizon $h$ accordingly to $d$ and evaluate it on 3 different budgets $T={2d, 4d, 6d}$. For the other baselines, we use the default shipped hyperparameter configuration from the authors for \textsc{PushT} and \textsc{Cube}. As these baselines don't include checkpoints for \textsc{TwoRoom} and \textsc{Reacher}, we adopt their codebase for training them. We then do a hyperparameter sweep, following their paper, to evaluate these checkpoints and report the best performing result on 3 budgets $T={d, 2d, 4d}$. Lastly, for \method, we sweep $h$ and $H$. Each $(h,H)$ cell is run at a max budget of $T=4d$ and success rate is read off every smaller $T$. The best $(h,H)$ moves with $T$, so we report the non-dominated set of this $(h, H, T)$ search in success against wall clock latency.

Complete details of all model, baselines, benchmarks, and hyperparameters to faciliate reproducibility can be found in Appendix Section~\ref{sec:appendix-setup} \change{and Section~\ref{app:hjepa}}.

\subsection{Main Result}
\label{sec:results-main}

\begin{figure}[t]
  \centering
  \includegraphics[width=\linewidth]{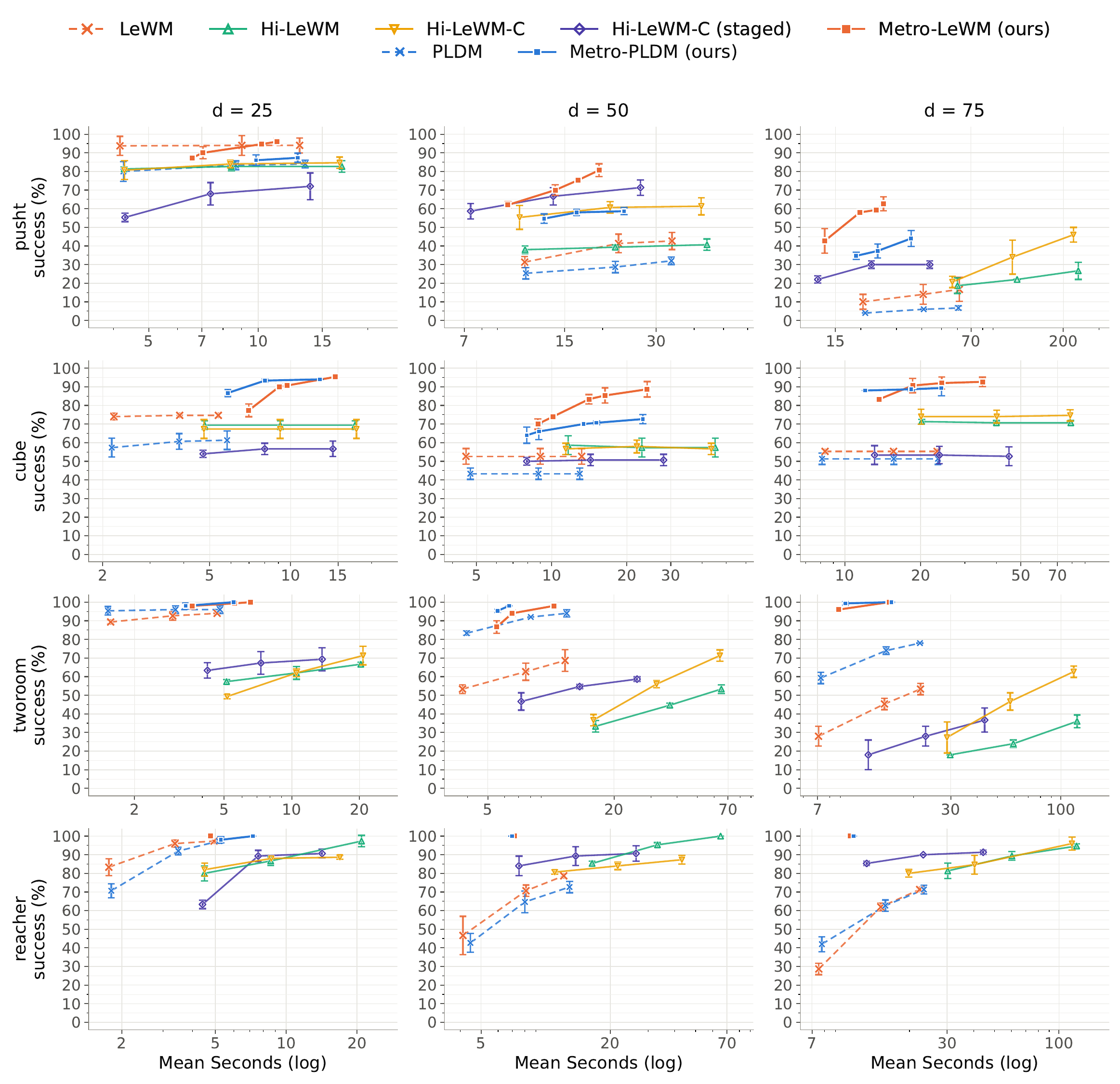}
  \caption{Performance and cost comparison between \method and baselines across varying budgets. Four benchmarks $\times$ three goal distances (planning horizons) $d$.}
  \label{fig:pareto}
\end{figure}

Figure~\ref{fig:pareto} plots success against online wall clock for all four benchmarks and three goal distances. All methods are evaluated with a range of planning budgets, defining a pareto front per method.

\keypoint{Against hierarchical baselines.}
At the longest horizon ($d=75$), \methodlewm's best configuration beats the strongest hierarchical baseline (Hi-LeWM-C) by $+16.67$, $+18.00$, $+37.33$ and $+4.00$ points on \textsc{PushT}, \textsc{Cube}, \textsc{TwoRoom} and \textsc{Reacher} respectively, while running $2.2\text{--}10.9\times$ faster. Across all goal distances, it matches or beats every hierarchical baseline. \methodpldm does the same everywhere except \textsc{PushT} at $d=50$. Only the hierarchical baselines' cheapest, less successful configurations ever run faster than \method's cheapest, most visibly at shorter horizons.

\keypoint{Against flat planning.} Both \methodlewm and \methodpldm beat flat planning on the same base world model in all cases, and the margin grows with $d$. For instance, at $d=75$, \methodlewm adds $+28.67$--$46.67$ over LeWM, while at $d=25$, the gain is small except on \textsc{Cube}. This is expected as the graph pays off when the task is long, which is where a hierarchy is supposed to help.

\keypoint{Overall.} In the target long horizon ($d=75$) regime, {\method}s are substantially faster and stronger than competitors, reaching 100\% success on long horizon \textsc{TwoRoom} and \textsc{Reacher}. 

\keypoint{Offline Costs.} 
Table~\ref{tab:offline} shows the offline 'training' cost between \method and macro planner training. \method's graph costs $9$--$39$ minutes on one GPU and takes no gradient steps while building a macro planner costs $13$--$56\times$ that in GPU-time. Note that these training costs serve all three hierarchical baselines per benchmark. 

\begin{table}[t]
  \caption{One-time offline cost per benchmark. \method's ``encode'' is one
  forward pass of the frozen encoder over every recorded frame; ``index'' is
  the exact nearest-neighbour table plus the edge-cost computation. Hi-LeWM-C's training
  figures are wall clock at the original authors'
  recorded recipe; GPU-hours are wall clock $\times$ devices.}
  \label{tab:offline}
  \centering
   \begin{tabular}{lccc ccc}  
    \toprule
    & \multicolumn{3}{c}{\method preprocessing (1 GPU)} & \multicolumn{3}{c}{Hi-LeWM-C training (4 GPUs)} \\
    \cmidrule(lr){2-4}\cmidrule(lr){5-7}
    benchmark & encode & index & \textbf{total} & grad.\ steps & wall & \textbf{GPU-hours} \\
    \midrule
    \textsc{PushT}   & $17.1$ & $6.0$ & $\mathbf{23.1}$\,min & $\phantom{0}76{,}215$  & $4.01$\,h & $\mathbf{16.0}$ \\
    \textsc{Cube}    & $34.2$ & $4.7$ & $\mathbf{38.9}$\,min & $118{,}125$ & $9.09$\,h & $\mathbf{36.4}$ \\
    \textsc{TwoRoom} & $\phantom{0}7.6$  & $1.4$ & $\mathbf{\phantom{0}9.0}$\,min  & $\phantom{00}8{,}370$   & $0.50$\,h & $\mathbf{\phantom{0}2.0}$  \\
    \textsc{Reacher} & $24.4$ & $4.7$ & $\mathbf{29.1}$\,min & $118{,}125$ & $6.58$\,h & $\mathbf{26.3}$ \\
    \bottomrule
  \end{tabular}
\end{table}

\begin{table}[h]
  \caption{\change{Success (\%) in the H-JEPA protocol mean $\pm$ standard deviation over three seeds. H-JEPA-$L$ and HWM-$L$ denote hierarchies with $L$ levels.}}
  \label{tab:hjepa}
  \centering
  \begin{tabular}{l cccc}
    \toprule
    Approach & \textsc{Cube} & \textsc{PushT} & \textsc{AntMaze} & \textsc{FourRoom} \\
    \midrule
    LeWM        & $25.33 \pm \phantom{0}4.16$  & $38.00 \pm \phantom{0}5.29$ & $16.00 \pm \phantom{0}5.29$  & $46.00 \pm 15.62$ \\
    H-JEPA-2    & $43.33 \pm 19.73$ & $41.33 \pm \phantom{0}6.43$ & $40.00 \pm 11.14$ & $80.67 \pm 11.02$ \\
    H-JEPA-3    & $48.67 \pm \phantom{0}9.02$  & $\phantom{0}9.33 \pm \phantom{0}4.16$  & $67.33 \pm \phantom{0}6.43$  & $96.00 \pm \phantom{0}2.00$ \\
    H-JEPA-4    & $56.67 \pm \phantom{0}6.43$  & $\phantom{0}0.00 \pm \phantom{0}0.00$  & $60.00 \pm \phantom{0}0.00$  & $84.00 \pm \phantom{0}6.00$ \\
    HWM-2       & $48.67 \pm 10.26$ & $49.33 \pm \phantom{0}4.16$ & $33.33 \pm \phantom{0}4.62$  & $70.00 \pm \phantom{0}9.17$ \\
    HWM-3       & $57.33 \pm \phantom{0}4.16$  & $\phantom{0}8.67 \pm \phantom{0}4.16$  & $32.67 \pm \phantom{0}5.03$  & $98.00 \pm \phantom{0}2.00$ \\
    HWM-4       & $34.67 \pm 12.86$ & $\phantom{0}0.00 \pm \phantom{0}0.00$  & $18.67 \pm \phantom{0}3.06$  & $94.00 \pm \phantom{0}6.00$ \\
    \methodlewm & $\mathbf{64.00 \pm \phantom{0}7.21}$ & $\mathbf{52.67 \pm 4.16}$ & $\mathbf{81.33 \pm 11.72}$ & $\mathbf{98.67 \pm 1.15}$ \\
    \bottomrule
  \end{tabular}
\end{table}

\subsection{\change{Comparison in the H-JEPA Protocol}}
\label{sec:results-hjepa}

\change{H-JEPA~\citep{zhang2026hjepa}, concurrent with our work, recently releases flat LeWM checkpoints together with H-JEPA and HWM~\citep{zhang2026hierarchical} of $2-4$ levels (H-JEPA-$L$ and HWM-$L$ for $L$ levels), and a gradient-based planner on OGBench-\textsc{Cube}, \textsc{PushT}, Visual \textsc{AntMaze}~\citep{park2025ogbench}, and \textsc{FourRoom} Distractors~\citep{zhang2026hjepa}. We plug the provided flat LeWM into \method (\methodlewm) and  compare with H-JEPA and HWM under their evaluation setup: same planner test tasks, goal distances, budgets, seeds as detailed in Appendix~\ref{app:hjepa-setup}. Unlike LeWM's CEM configuration in Section~\ref{sec:results-main}, which executes all $h$ planned blocks before replanning, H-JEPA's flat planner executes a fixed number of blocks per plan and decays its planning horizon over the episode. As \method's sub-goal lies at most $H$ frames ahead throughout the episode, we keep the horizon fixed at $h$ and tune the number of executed blocks together with $(h,H)$ on tuning tasks disjoint from the test set.}

\change{We rerun all baselines with H-JEPA's released code and checkpoints, which largely reproduces their reported results (within two standard errors in 27 of 28 cells). As in those results, the best depth depends on the hierarchy and the benchmark, e.g. for H-JEPA, it is three levels on \textsc{AntMaze} and \textsc{FourRoom}, four on \textsc{Cube}, and two on \textsc{PushT}. 
Table~\ref{tab:hjepa} shows that \methodlewm matches or exceeds the best hierarchy on every benchmark, by the largest margin on \textsc{AntMaze} and within one standard deviation on the other three, without a depth choice or training a macro level.}

\subsection{Analysis}
\label{sec:results-analysis}

In this section, we show further analyses by picking each benchmark's highest-success $d=75$ result \change{of Section~\ref{sec:results-main}} and running several ablations on them. We also experiment with fine-tuning \change{the world model} and plugging it into \method instead of using the provided frozen micro planners and show that it leads to a boost in performance in Appendix Section~\ref{app:stronger_micro}.

\begin{table}[t]
  \caption{Success (\%) with and without graph re-entry, $d = 75$.}
  \label{tab:reentry}
  \centering
   \begin{tabular}{l cc @{\hskip 1.2em} cc}
    \toprule
    & \multicolumn{2}{c}{\methodlewm} & \multicolumn{2}{c}{\methodpldm} \\
    \cmidrule(lr){2-3}\cmidrule(lr){4-5}
    benchmark & re-entry & w/o re-entry & re-entry & w/o re-entry \\
    \midrule
    \textsc{PushT}   & $\mathbf{62.67 \pm 4.62}$           & $49.33 \pm 3.06$ & $\mathbf{44.00 \pm 5.29}$           & $39.33 \pm 6.43$ \\
    \textsc{Cube}    & $\mathbf{92.67 \pm 3.06}$  & $84.00 \pm 0.00$ & $\mathbf{89.33 \pm 5.03}$  & $88.67 \pm 2.31$ \\
    \textsc{TwoRoom} & $\mathbf{100.00 \pm 0.00}$ & $76.67 \pm 4.16$ & $\mathbf{100.00 \pm 0.00}$ & $61.33 \pm 5.77$ \\
    \textsc{Reacher} & $\mathbf{100.00 \pm 0.00}$ & $99.33 \pm 1.15$ & $\mathbf{100.00 \pm 0.00}$ & $98.00 \pm 0.00$ \\
    \bottomrule
  \end{tabular}
\end{table}

\keypoint{Impact of re-entry.} \method re-enters the graph at the start of every macro step (Eq.~\ref{eq:entry}), so wherever the micro planner lands is a legal entry point with its own optimal remaining route. The alternative is to commit by solving the route once at the beginning and providing the micro planner sub-goals along that route regardless of where it lands. Table~\ref{tab:reentry} shows a drop in performance for committed routes, which is expected as the committed route has no way to re-cost the error the micro planner leaves behind. \textsc{Reacher} is an exception where re-entry has negligible difference as its environment is collected using a uniform random policy and, in majority of cases, the goal is often reachable in $<= 2$ macro plans, leaving little room for compounding errors.

\begin{table}[t]
  \caption{Comparison with a privileged \textsc{oracle}, where the sub-goals come from the query's episode own trajectory at inference.
  Arrival ratio is the median, over the successful episodes, of
  frames-to-success divided by the expert demonstration's own first-hitting time;
  $<1$ means arriving \emph{sooner than the demonstration}.}
  \label{tab:oracle}
  \centering
    \begin{tabular}{ll cc @{\hskip 1.2em} cc}
    \toprule
    & & \multicolumn{2}{c}{Success (\%)} & \multicolumn{2}{c}{Arrival Ratio} \\
    \cmidrule(lr){3-4}\cmidrule(lr){5-6}
    Approach & Benchmark & \method & \textsc{oracle} & \method & \textsc{oracle} \\
    \midrule
    \methodlewm & \textsc{PushT}   & $62.67 \pm 4.62$           & $\mathbf{64.00 \pm 2.00}$ & $\mathbf{2.71}$          & $2.89$ \\
         & \textsc{Cube}    & $\mathbf{92.67 \pm 3.06}$  & $72.00 \pm 4.00$          & $\mathbf{0.44}$ & $1.32$ \\
         & \textsc{TwoRoom} & $\mathbf{100.00 \pm 0.00}$ & $96.67 \pm 2.31$          & $\mathbf{0.93}$ & $1.65$ \\
         & \textsc{Reacher} & $\mathbf{100.00 \pm 0.00}$ & $93.33 \pm 3.06$          & $\mathbf{0.65}$ & $2.38$ \\
    \midrule
    \methodpldm & \textsc{PushT}   & $44.00 \pm 5.29$           & $\mathbf{72.00 \pm 4.00}$ & $\mathbf{3.08}$          & $3.26$ \\
         & \textsc{Cube}    & $\mathbf{89.33 \pm 5.03}$  & $75.33 \pm 1.15$          & $\mathbf{0.63}$ & $1.82$ \\
         & \textsc{TwoRoom} & $\mathbf{100.00 \pm 0.00}$ & $95.33 \pm 1.15$          & $\mathbf{0.57}$ & $1.10$ \\
         & \textsc{Reacher} & $\mathbf{100.00 \pm 0.00}$ & $96.00 \pm 0.00$          & $\mathbf{0.22}$ & $0.75$ \\
    \bottomrule
  \end{tabular}

\end{table}

\keypoint{Comparison with expert's own path.} To ask whether the route is effective, we replace the graph with a privileged \textsc{oracle}: the query episode's own demonstration trajectory, $m^{\star}$, read at inference, aimed one hop ahead and localised to the state actually read. In other words, the \textsc{oracle} \textit{1)} provides the micro planner with a sub-goal, retrieved from $m^{\star}$, that is one micro plan hop away from its current state, and \textit{2)} finds the nearest state, in $m^{\star}$, to where the planner lands and retrieves the next sub-goal from there. 

Table~\ref{tab:oracle} shows both the success rate and the arrival ratio for both \method and \textsc{oracle}. The arrival ratio is the median of the number of frames that the approach took divided by the number of frames that the expert demonstration took to reach the goal over all successful episodes. Hence, an arrival ratio of $<1$ means the route \method took is shorter than the route the demonstration took in the dataset. In particular, for \textsc{Cube}, \textsc{Reacher}, and \textsc{TwoRoom}, \method reaches the goal sooner than the query's demonstration through stitching multiple demonstrations to form a shorter path. This is largely possible because the demonstrations in this benchmark tend to wander instead of taking the shortest path straight to the goal, \textit{e.g.} in \textsc{Cube}, the robotic arm often lingers or hovers around before picking up the cube. In contrast, demonstrations in \textsc{PushT} are more likely to take the shortest path to push the T block to its goal; both \method and \textsc{oracle} took a longer path.

\begin{table}[t]
  \caption{Success (\%) against corpus size, $d = 75$. $1\%$ is 187
  episodes on \textsc{PushT} and 100 elsewhere.}
  \label{tab:sparsity}
  \centering
    \begin{tabular}{ll cccc}
    \toprule
    Approach & Corpus Size & \textsc{Cube} & \textsc{PushT} & \textsc{Reacher} & \textsc{TwoRoom} \\
    \midrule
    \methodlewm & $100\%$ & $92.67 \pm 3.06$ & $62.67 \pm 4.62$ & $100.00 \pm 0.00$ & $100.00 \pm 0.00$ \\
         & $10\%$  & $90.00 \pm 5.29$ & $46.00 \pm 2.00$ & $98.67 \pm 1.15$  & $100.00 \pm 0.00$ \\
         & $3\%$   & $79.33 \pm 1.15$ & $38.00 \pm 2.00$ & $100.00 \pm 0.00$ & $100.00 \pm 0.00$ \\
         & $1\%$   & $72.00 \pm 2.00$ & $31.33 \pm 4.16$ & $100.00 \pm 0.00$ & $100.00 \pm 0.00$ \\
    \midrule
    \methodpldm & $100\%$ & $89.33 \pm 5.03$ & $44.00 \pm 5.29$ & $100.00 \pm 0.00$ & $100.00 \pm 0.00$ \\
         & $10\%$  & $90.67 \pm 1.15$ & $29.33 \pm 5.03$ & $100.00 \pm 0.00$ & $100.00 \pm 0.00$ \\
         & $3\%$   & $84.67 \pm 1.15$ & $26.67 \pm 5.03$ & $100.00 \pm 0.00$ & $95.33 \pm 1.15$ \\
         & $1\%$   & $81.33 \pm 1.15$ & $10.67 \pm 3.06$ & $100.00 \pm 0.00$ & $95.33 \pm 3.06$ \\
    \bottomrule
  \end{tabular}

\end{table}

\keypoint{Dataset density.} \method's success hinges on the density of the dataset, enabling it to stitch multiple episodes together to form a path to the goal. Since collecting training demonstrations is notoriously expensive and time-consuming in real-world applications, we question the effectiveness of \method under sparser conditions. To this end, we rebuild the graph from $1\%$, $3\%$, and $10\%$ of the recorded episodes and show results of \method in Table~\ref{tab:sparsity}. In particular, results for \textsc{TwoRoom} and \textsc{Reacher} shows that a hundred recorded episodes is sufficient to stitch a route that solves the task every time. Success rates for \textsc{PushT} and \textsc{Cube}, on the other hand, drop as the number of demonstrations drop. Nonetheless, at $1\%$, \methodlewm's performance is still similar to the performance of the next best hierarchical baseline at their shipped solver budget ($T=2d$), which was trained on the whole corpus ($31.33 \pm 4.16$ against $34.00 \pm 9.17$ for \textsc{PushT} and $72.00 \pm 2.00$ against $74.00 \pm 3.46$ for \textsc{Cube}).

\section{Conclusion}
\label{sec:conclusion}

We identify physically unrealisable sub-goals as a critical flaw in hierarchical latent planning and introduce a decoder-free realisability residual to quantify it. \method avoids this problem by framing long-horizon planning as shortest-path search over a dataset graph whose edge costs are measured in frames. This guarantees the physical plausibility of its sub-goals and mitigate micro planner execution errors through dynamic graph re-entry. Across different benchmarks, \method improves long-horizon success and often finds paths shorter than the offline demonstrations. \change{Under the protocol of the concurrent H-JEPA, it also matches or exceeds hierarchies of every depth.}
\change{While \method handles layout variation that recurs across the corpus, e.g. the closed door that changes from episode to episode in \textsc{FourRoom}, it} relies on bridge edges to connect visually similar states across episodes\change{. We therefore} anticipate it would struggle in environments \change{whose layouts rarely repeat, as} stitching trajectories across fundamentally different scenes becomes structurally infeasible. While extending retrieval-based graph planning to accommodate such topological diversity remains an exciting direction for future work, \method ultimately establishes a highly effective, training-free foundation for robust long-horizon planning.

\newpage
\bibliographystyle{unsrtnat}

\appendix
\newpage
\begin{figure}[t]
  \centering
  \includegraphics[width=\linewidth]{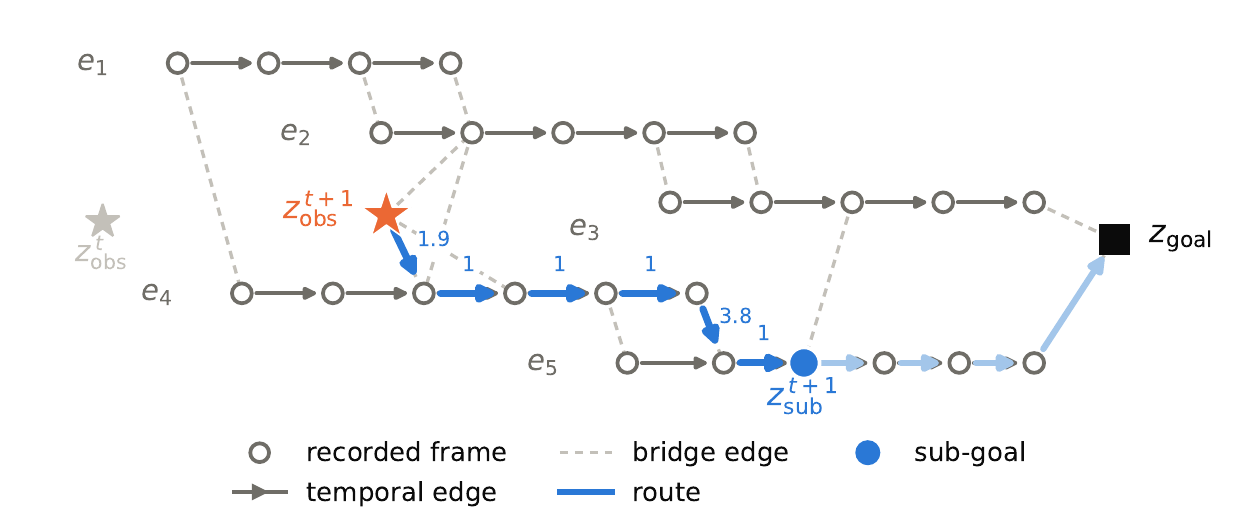}
  \caption{\method's graph re-entry. This follows Fig.~\ref{fig:steps}: after the micro planner stops at $\zobs^{t+1}$, the next macro step re-enters the graph there following the optimal shortest route from $\zobs^{t+1}$ to $\zgoal$.}
  \label{fig:reentry}
\end{figure}

\section{Graph entry}
\label{app:proofs}\label{app:reentry}

In Eq.~\ref{eq:entry}, we evaluate the optimal path from $\zobs$ to $\zgoal$ by dynamically adding $\zobs$ in the graph and applying a one-step Bellman backup over its $k$-nearest neighbours. Figs.~\ref{fig:steps} and \ref{fig:reentry} illustrates this. Below is the proof, for completeness, that this backup yields the exact shortest-path from $\zobs$ to $\zgoal$.

\begin{proposition}[Exactness of the entry backup]
\label{prop:entry}
Let $G^{\star}$ be a weighted directed graph containing a goal vertex $g$, and let $J(v)$ denote the exact shortest-path cost from any vertex $v \in G^{\star}$ to $g$. If a new vertex $\zobs$ is added to the graph with out-edges to a candidate set $\mathcal{N} \subseteq G^{\star}$ and no in-edges, then Eq.~\ref{eq:entry} computes the exact shortest-path cost from $\zobs$ to $g$, and the original costs $J(v)$ remain unchanged.
\end{proposition}

\begin{proof}
Because $\zobs$ has no in-edges, no path originating within $G^{\star}$ can visit it. Therefore, the addition of $\zobs$ cannot create any new shortcuts between existing vertices, leaving $J(v)$ unchanged for all $v \in G^{\star}$. 

Furthermore, any path from $\zobs$ to $g$ must begin with an edge $\zobs \to u$ for some $u \in \mathcal{N}$. Since the path cannot revisit $\zobs$, the remainder of the route must lie entirely within $G^{\star}$. By definition, the minimal cost of this remaining route is $J(u)$. Thus, taking the minimum over all out-edges $u \in \mathcal{N}$ yields the exact shortest-path cost from $\zobs$ to $g$.
\end{proof}

Because Proposition~\ref{prop:entry} holds for \emph{any} candidate set $\mathcal{N}$, our entry rule introduces no approximation error. The one-step backup computes the true mathematical optimum for the graph it is run on; tuning the number of nearest neighbors ($k$) simply determines how densely $\zobs$ is connected to that graph.

\section{Realisability search of macro latent sub-goals}
\label{sec:appendix-realizability}

In Section~\ref{sec:realizability}, we show that latent sub-goals from existing hierarchical works are not realisable. In this Section, we explain the search in detail. The parameterisation of $\mathcal{S}$ is the recorded state itself, i.e. on \textsc{PushT}, it is $5$ dimensions, including the agent's $xy$, block's $xy$, and its angle. We ran the search over the simulator's state space for each sub-goal produced by Hi-LeWM-C (Eq.~\ref{eq:rz}) at two different compute budgets: 1) 3 CEM searches in parallel with 20 iterations of 384 candidates per iteration (23040 renders per sub-goal) and 2) 8 CEM searches with 30 iterations of 384 candidates per iteration (92160 renders per sub-goal). Despite increasing compute budget by fourfold, the median $r(z)$ barely drops $r(z)$: $0.92 \rightarrow 0.89$ for \textsc{Cube} and $2.31 \rightarrow 2.29$ for \textsc{PushT}. The exact $r(z)$ doesn't mean much on its own since the latent space has no absolute scale. However, the fact that it plateaus above $0$ tells us the optimization has failed to find a plausible solution.

\change{We repeat the search above on the sub-goals of the best-performing depth of H-JEPA and HWM on \textsc{Cube} and \textsc{PushT}, shown in Figure~\ref{fig:realizability-hjepa-cube} and \ref{fig:realizability-hjepa-pusht} respectively. Likewise, increasing the compute budget fourfold barely lowers the median $r(z)$: $0.98 \rightarrow 0.95$ for H-JEPA-4 and $0.74 \rightarrow 0.68$ for HWM-3 on \textsc{Cube}, and $1.51 \rightarrow 1.50$ for H-JEPA-2 and $1.19 \rightarrow 1.17$ for HWM-2 on \textsc{PushT}. }

\begin{figure}[h]
  \centering
  \includegraphics[width=\linewidth]{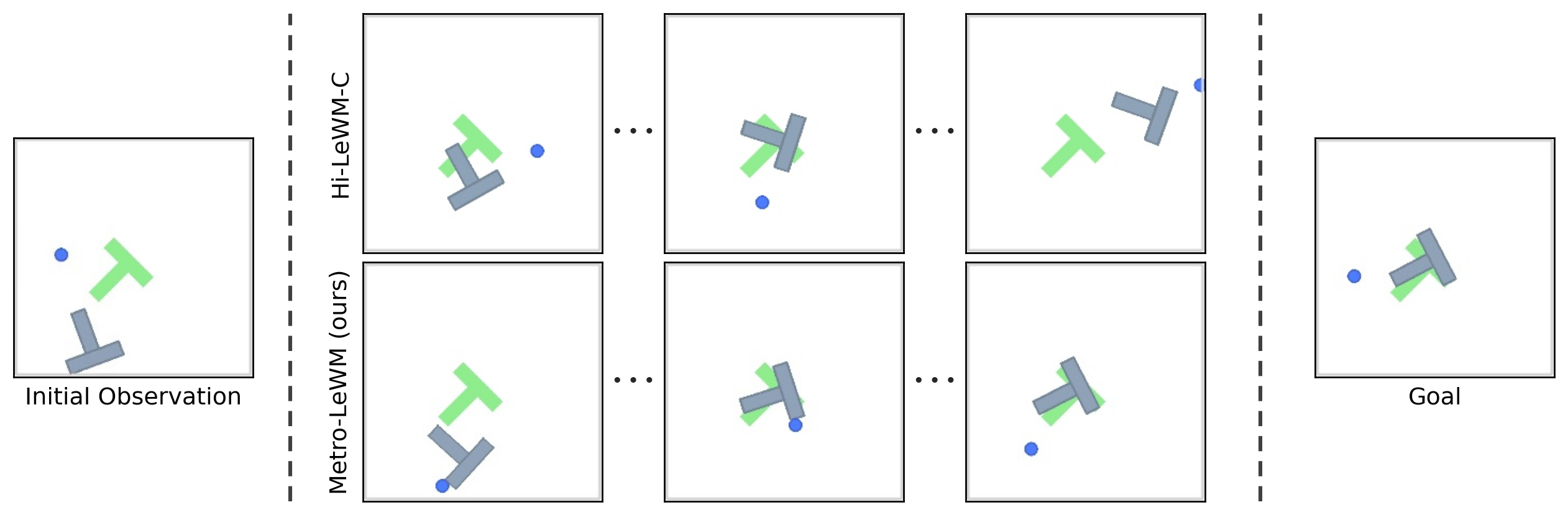}
  \caption{Top row: closest realisable Hi-LeWM-C’s sub-goals from its macro planner found by Eq. 3
and drawn by the environment’s own renderer. Bottom row: our method’s, METRO-WM, sub-goals
which are retrieved real states of other demonstrations in the dataset.}
  \label{fig:realizability-pusht}
\end{figure}

\begin{figure}[h]
  \centering
  \includegraphics[width=\linewidth]{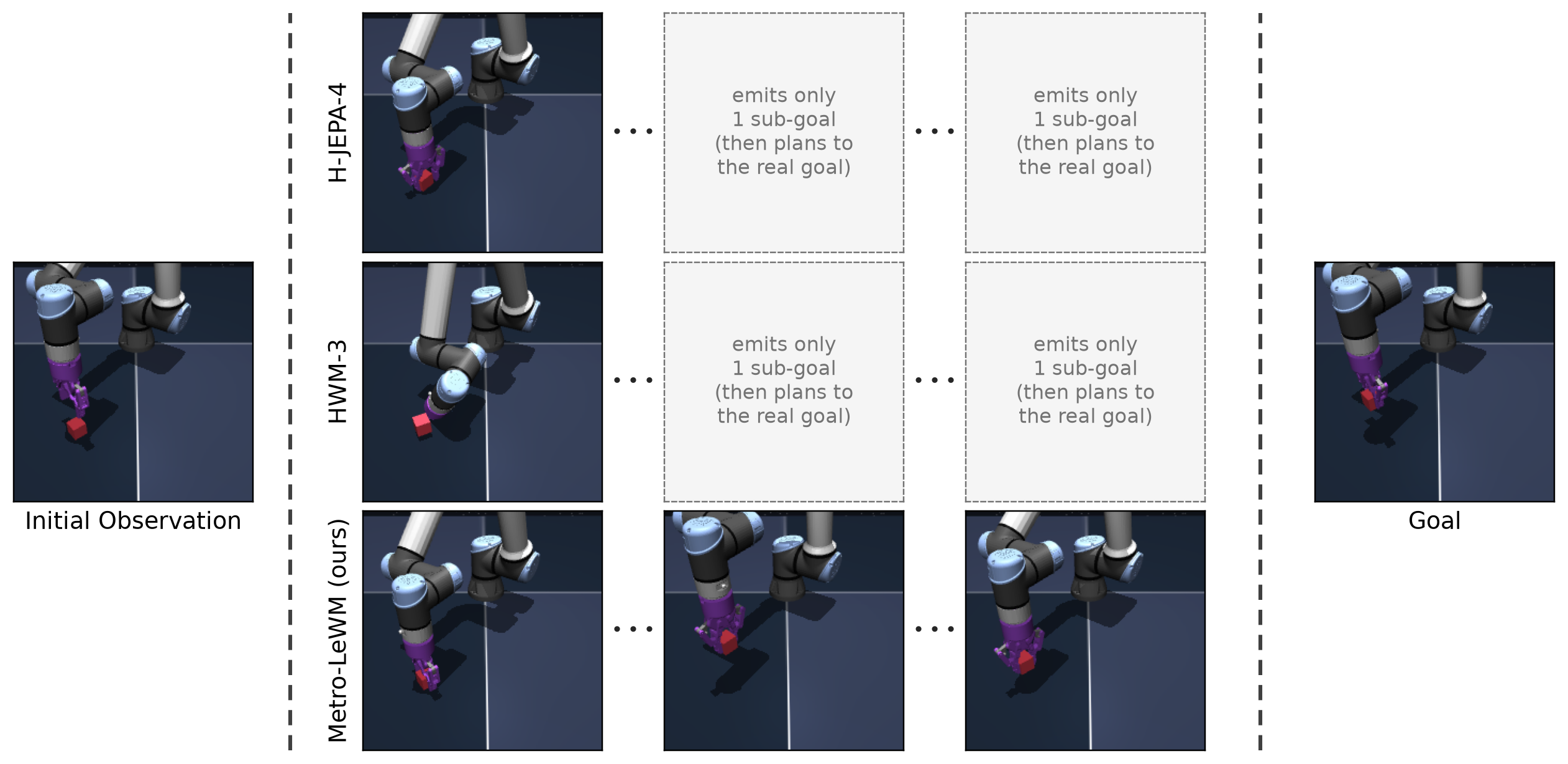}
  \caption{\change{Top two rows: closest realisable states to the sub-goals emitted by H-JEPA-4 and HWM-3 on \textsc{Cube}, drawn by the environment's own renderer; both emit a single sub-goal per task before planning to the goal. Bottom row: our \method sub-goals, retrieved real states.}}
  \label{fig:realizability-hjepa-cube}
\end{figure}

\begin{figure}[h]
  \centering
  \includegraphics[width=\linewidth]{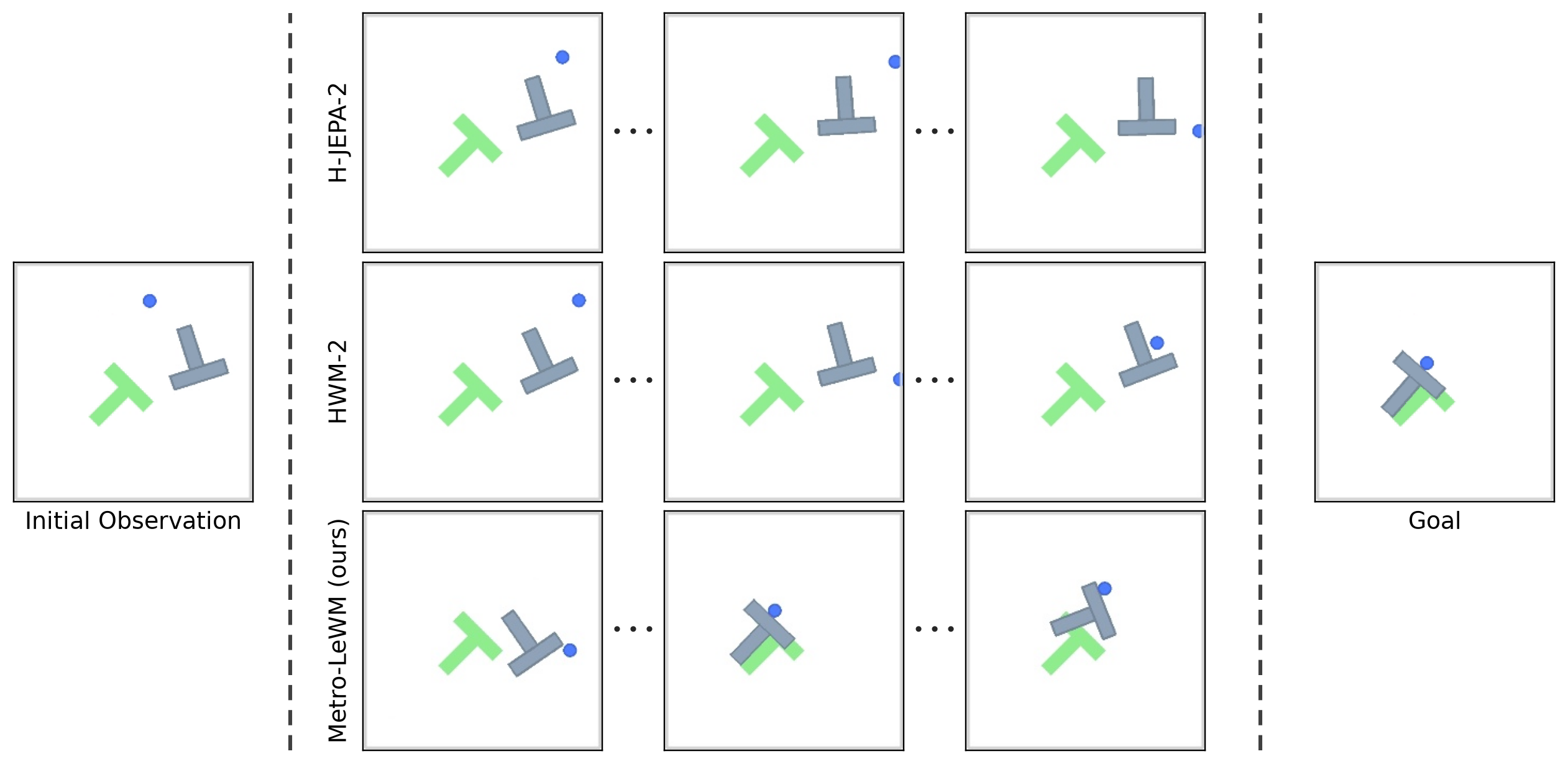}
  \caption{\change{Top two rows: closest realisable states to the sub-goals emitted by H-JEPA-2 and HWM-2 on \textsc{PushT}, drawn by the environment's own renderer. Bottom row: our \method sub-goals, retrieved real states.}}
  \label{fig:realizability-hjepa-pusht}
\end{figure}

\section{\method's Algorithm}
\label{app:algorithm}

\begin{algorithm}[H]
\caption{\method: sub-goal selection by shortest paths over recorded frames.
Lines 1--2 are offline and shared by every episode; lines 3--4 run once per
episode; the loop runs at each macro replanning step.}
\label{alg:method}
\begin{algorithmic}[1]
\REQUIRE frozen encoder $\phi$; corpus $\Dset$; micro planner $\pilo$; waypoint
         budget $H$; neighbours $k$; pricing quantile $\alpha$
\STATE \textbf{offline, once:} encode $\Dset$; estimate the pricing curve $\Lambda$
        and set $\varepsilon \leftarrow \Lambda(H)$
\STATE \textbf{offline, once:} build $E_{\mathrm{tmp}}$ and $E_{\mathrm{brg}}$
       (Eq.~\ref{eq:bridge}), symmetrizing bridges by elementwise max and
       weighting each by $\kappa$ (Eq.~\ref{eq:kappa})
\STATE \textbf{per episode:} $\zgoal \leftarrow \phi(o_{\mathrm{goal}})$;
       delete the \change{query's episode} $m^{\star}$ \change{if it is in the graph}; insert $g$ with $k$ priced
       bridges
\STATE \textbf{per episode:} run Dijkstra's algorithm from $g$
       $\rightarrow$ $J(u)$, $\mathrm{succ}(u)$ at \emph{every} $u$
       (Eq.~\ref{eq:dijkstra})
\WHILE{environment steps remain and goal not reached}
  \STATE $\zobs \leftarrow \phi(o_t)$
        \COMMENT{re-encode the state actually reached}
  \STATE $\mathcal{N} \leftarrow \mathrm{kNN}_{k}(\zobs)$, excluding
         trajectory $m^{\star}$
  \IF{$J(u) = \infty$ for all admissible $u \in \mathcal{N}$}
    \STATE $\zsub \leftarrow$ nearest-to-goal candidate if it improves on
           $\zobs$, else $\zgoal$ \COMMENT{no route}
  \ELSE
    \STATE $u^{\star}, w(\zobs \to u_0), J(\zobs) \leftarrow$ Bellman entry backup over
           $\mathcal{N}$ (Eq.~\ref{eq:entry})
    \IF{$J(\zobs) \le H$}
      \STATE $\zsub \leftarrow \zgoal$
             \COMMENT{goal within one waypoint budget}
    \ELSE
      \STATE $\zsub \leftarrow z_{u_{\ell^{\star}}}$, walking $\mathrm{succ}$
             from $u^{\star}$ on the remaining budget $H - w(\zobs \to u_0)$, taking the
             first hop unconditionally (Eq.~\ref{eq:walk})
    \ENDIF
  \ENDIF
  \STATE run $\pilo(\zobs, \zsub)$
         \COMMENT{run micro planner given sub-goal}
\ENDWHILE
\end{algorithmic}
\end{algorithm}

\section{Latent metric saturation}
\label{app:saturation}

\begin{figure}[t]
  \centering
  \includegraphics[width=\linewidth]{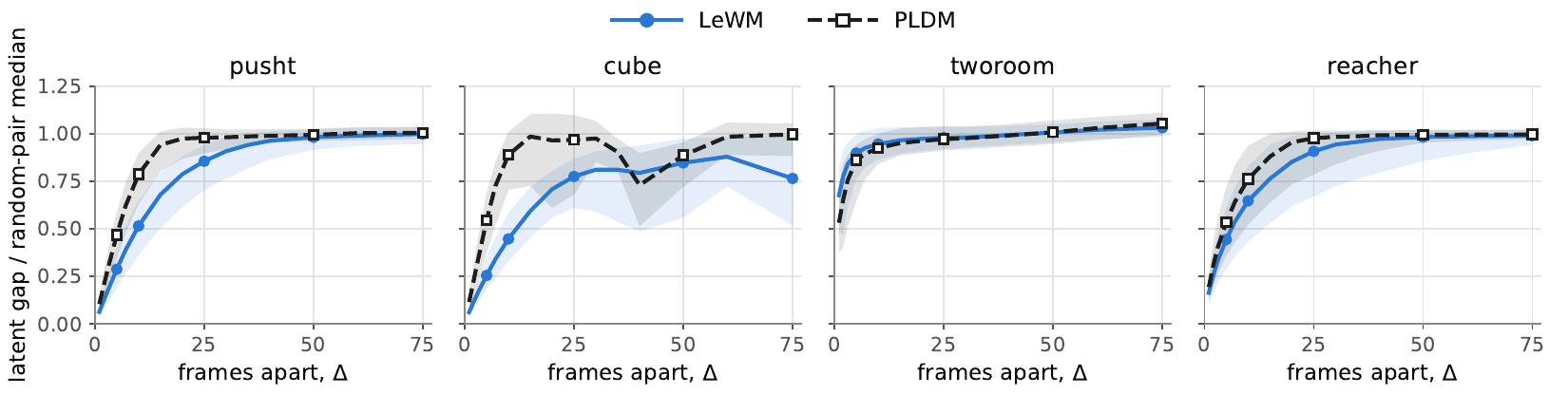}
  \caption{\textbf{Latent distance stops resolving time within a few dozen
  frames.} Median $\lVert z_t - z_{t-d}\rVert$ between two frames of the same
  demonstration $d$ frames apart (band: interquartile range), divided by the
  median distance between frames of two \emph{different}, randomly drawn
  demonstrations. A curve at $1$ cannot tell a state $d$ frames
  away from a random frame of the corpus. }
  \label{fig:saturation}
\end{figure}

Section~\ref{sec:prelim} and \ref{sec:method} assert that the latent metric saturates and that this is a property of the representation rather than of the rollout. We provide evidence for this claim in Fig.~\ref{fig:saturation}, which plots the median, and IQR, latent gap between frames within the same demonstration divided by the latent gap between the same source frame to a random frame from a different demonstration over $5,000$ samples. This shows that the latent distance can only tell frames apart at small gaps, and how small depends on the benchmark and encoder. For instance, for \textsc{PushT}, between $d=5$ and $d=10$, y-axis moves from $0.29 \rightarrow 0.52$, indicating that a bigger gap in frames gives a bigger latent distance and the CEM planner can tell them apart. In contrast, between $d=50$ and $d=75$, y-axis moves from $0.98 \rightarrow 1.0$ - the cost can't say which is closer. This, on top of compounding prediction errors, is why the micro planner fails over longer horizons and why every edge in \method's graph is priced in frames.

\section{Additional Results}

\subsection{What a bridge looks like}

\begin{figure}[H]
  \centering
  \includegraphics[width=\linewidth]{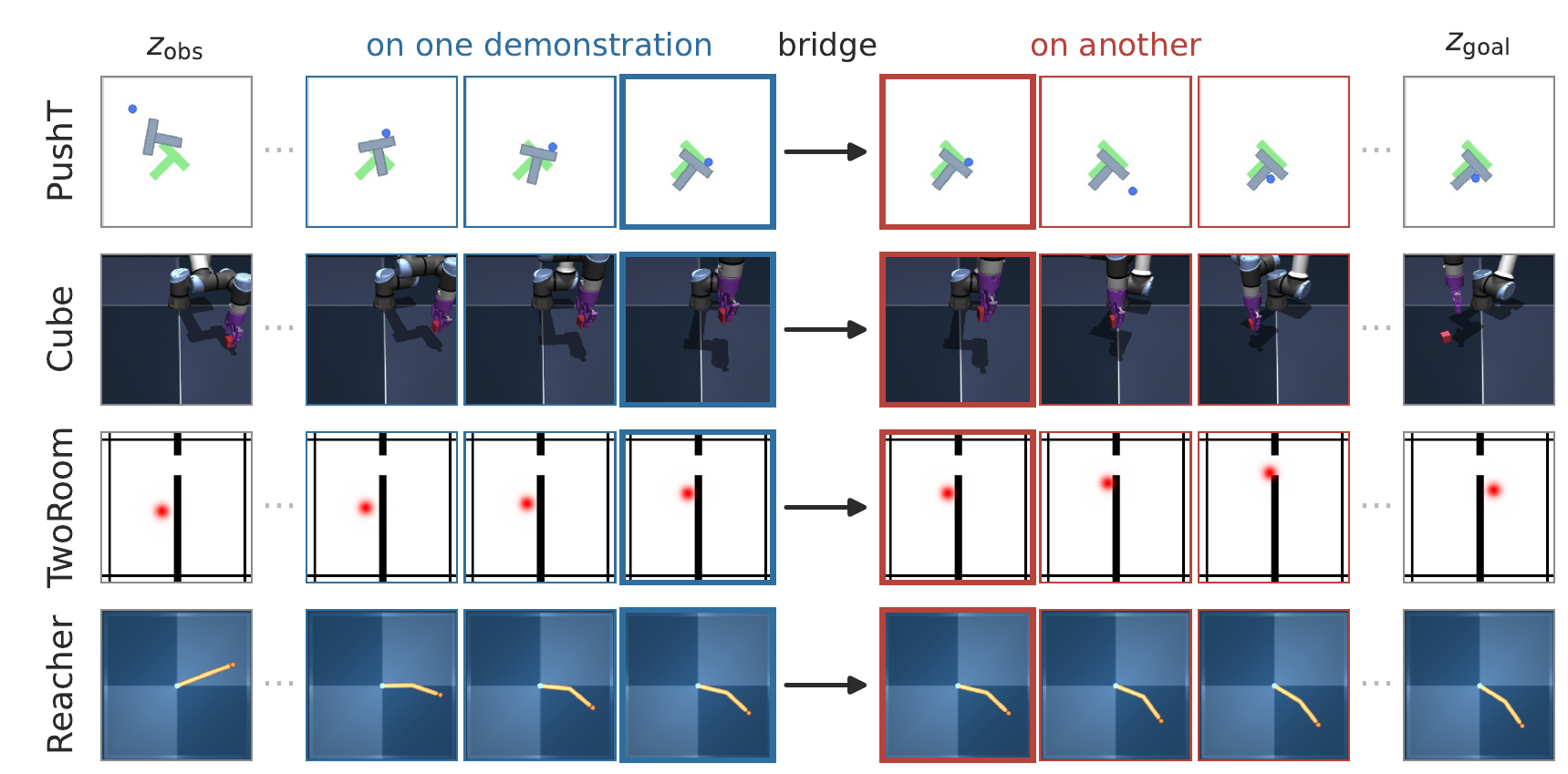}
  \caption{\textbf{A bridge, in pixels, on each benchmark.} Each row is part of
  one route computed by \methodlewm for one query of the evaluation set. The ends are the query's
  $\zobs$ and $\zgoal$ and every panel between them is a recorded frame the route
  traverses, and $\cdots$ marks the stretches not shown. The route runs along
  one demonstration (blue), crosses a bridge, and continues along a different
  one (red). The two heavily outlined frames are that bridge's endpoints, which comes from different demonstrations in the dataset.}
  \label{fig:stitch}
\end{figure}

\begin{table}[tp]
  \caption{We fine-tune \textsc{PushT} LeWM, dubbed LeWM (ft) to get a stronger $d=25$ model and then evaluate it at $d = 75$. $\Delta$ is the gain over performance of the released LeWM weights, on the same $50$ episodes. }
  \label{tab:lowlevel}
  \centering
  \begin{tabular}{lcc}
    \toprule
    Approach & Success (\%) & $\Delta$ \\
    \midrule
    LeWM (ft) & $ 15.33 \pm 3.06 $ & $+1.33$ \\ 
    \methodlewm & $72.00 \pm 3.58$ & $+9.33$ \\
    \textsc{Oracle}  & $81.33 \pm 3.06$ & $+17.33$ \\
    \bottomrule
  \end{tabular}

\end{table}

\subsection{A stronger \change{world model}} 
\label{app:stronger_micro}

We hypothesise that an improvement to the micro planner converts into long-horizon success when plugged into \method. To this end, we further fine-tune the frozen LeWM, named LeWM (ft), improving its short horizon performance from $93.33 \pm 5.03$ to $97.33 \pm 1.15$ at $d=25$. We then evaluate this model on $d=75$ and show that despite a stronger $d=25$ success rate, $d=75$ flat planning leads to negligible differences in performance (Table~\ref{tab:lowlevel}). Plugging LeWM (ft) into \method and \textsc{oracle}, on the other hand, leads to substantial gains, indicating that macro sub-goal selection is critical to realizing micro-level gains.

\section{Extended related work}
\label{app:related}
Hierarchical planning tackles the compounding errors of long-horizon control by decomposing a task into shorter segments, with sub-goals either \emph{generated} by a trained model or \emph{retrieved} from recorded experience. Hierarchical latent world models generate their sub-goals; retrieval has been developed mainly in goal-conditioned reinforcement learning (RL) and visual navigation. \method brings retrieval to latent world-model planning, so that every sub-goal is, by construction, the encoding of a real state.

\paragraph{Sub-goal generation in latent world models.} Recent hierarchical latent world models generate sub-goals by searching at a coarser time scale~\citep{zhang2026hierarchical,caselli2026mindthegap}\change{, by training a stack of JEPAs end to end and planning top-down through it~\citep{zhang2026hjepa}}, or by predicting the next sub-goal embedding directly~\citep{masip2026ffjepa}. Nothing constrains these sub-goals to be encodings of real states. \citet{masip2026ffjepa} attributes some of its failures to predicted sub-goals that, once decoded, do not contain the actor, and \citet{caselli2026mindthegap} finds that hierarchical planning underperforms flat planning unless the search is restricted to macro-actions encoded from training demonstrations. Restricting the search does not restrict its output: Section~\ref{sec:realizability} shows that this variant still emits sub-goals that no state encodes to. Every sub-goal \method emits is, by construction, the encoding of a recorded frame.

\paragraph{Sub-goal generation in pixel space.} Outside JEPA-based world models, hierarchical visual planners generate visual sub-goals. Early approaches build on video prediction: time-agnostic prediction~\citep{jayaraman2019tap} lets a predictor select which intermediate frame to output, yielding predictable ``bottleneck'' frames that serve as sub-goals for visual MPC; HVF~\citep{nair2020hvf} generates sub-goal images and optimises them so that each segment is solvable by visual MPC with an action-conditioned video model; and GCP~\citep{pertsch2020lvd} recursively predicts intermediate observations for coarse-to-fine planning, then executes the plan with a trained inverse model. In a direct-pixel ablation, LEAP~\citep{nasiriany2019leap} finds image sub-goals outside the manifold of valid observations that exploit its value function. More recently, pretrained diffusion models generate sub-goals for learned low-level controllers: SuSIE~\citep{black2024susie} edits the current image into a language-conditioned sub-goal for a goal-conditioned policy, while UniPi~\citep{du2023unipi} and VLP~\citep{du2024vlp} generate video plans followed by an inverse-dynamics model or a goal-conditioned policy. For these generated visual sub-goals, realisability rests on the generative prior, and visual realism does not imply physical plausibility~\citep{motamed2026physicsiq}. \method instead retrieves its sub-goals from recorded frames.

\paragraph{Retrieval-based sub-goals and graph planning.} Sub-goal retrieval is mostly studied in reinforcement learning (RL) settings. To retrieve sub-goals, graph-based methods build a graph over stored states and plan on it with shortest paths, both in visual navigation~\citep{savinov2018sptm,shah2021ving} and in goal-conditioned reinforcement learning (RL)~\citep{eysenbach2019sorb,emmons2020sgm,huang2019landmarks,zhang2021l3p,baek2025gas}. When nodes are recorded states, every sub-goal is realisable; methods whose nodes are cluster centroids or generated images~\citep{zhang2021l3p,baek2025gas,liu2020htm} give up this guarantee. Most works read them off a goal-conditioned value or Q-function~\citep{eysenbach2019sorb,emmons2020sgm,huang2019landmarks,zhang2021l3p,kim2021higl,lee2022dhrl}; others train a distance, as a temporal-proximity classifier or contrastive score~\citep{savinov2018sptm,yang2020plan2vec,liu2020htm} or as a learned temporal distance~\citep{shah2021ving,baek2025gas}. VINN~\citep{pari2022vinn} likewise retrieves from demonstrations in a pretrained
embedding space, but retrieves actions for imitation rather than sub-goals for planning. Closest to our setting, TTGS~\citep{opryshko2026ttgs} adds no training of its own, but its edge weights come from the value function of an already-trained goal-conditioned agent.

\section{Detailed experimental setup}
\label{sec:appendix-setup}

For reproducibility and completeness, we provide all details of our experiments. We also open-source our code at https://github.com/SamsungLabs/metro-wm.
\change{Except in Section~\ref{sec:results-hjepa}, whose protocol is described in Appendix~\ref{app:hjepa}}, all benchmarks and frozen \change{world models} used are directly downloaded from https://github.com/lucas-maes/le-wm.

\subsection{Benchmarks}
\label{sec:appendix-benchmarks}

\begin{table}[h]
  \caption{Total number of demonstrations and frames for each benchmark in \change{the experiments of Sections~\ref{sec:results-main} and \ref{sec:results-analysis}}}
  \label{tab:benchmarks}
  \centering
  \begin{tabular}{lrr}
    \toprule
    benchmark & 
    no. of episodes/demonstrations & no. of frames  \\
    \midrule
    \textsc{PushT}   & $18{,}685$ & $2{,}261{,}996$ \\
    \textsc{Cube}    & $10{,}000$ & $1{,}970{,}000$ \\
    \textsc{TwoRoom} & $10{,}000$ & $880{,}809$     \\
    \textsc{Reacher} & $10{,}000$ & $1{,}970{,}000$ \\
    \bottomrule
  \end{tabular}
\end{table}

\paragraph{\textsc{PushT}.} A 2D task in which a circular agent pushes a
T-shaped block onto a target pose. The recorded state is 7-dimensional (agent
position, block position, block angle, agent velocity). 

\paragraph{\textsc{Cube} (OGBench, \texttt{cube\_single\_expert}).} A 7-DoF arm
with a parallel gripper moving a cube to a target pose, simulated in
MuJoCo~\citep{todorov2012mujoco}. Its realizable state space is 11-dimensional
(six arm joints, the gripper driver, cube position and cube yaw).

\paragraph{\textsc{TwoRoom}.} A navigation task through a doorway between two
rooms. Its episodes are short (92 frames), which matters only for the
hierarchical baselines' training cost: at a 90-frame training window it yields
$14\times$ fewer training samples than \textsc{Cube}.

\paragraph{\textsc{Reacher} (dm\_control).}
Joint-space reaching to a target configuration. The released corpus is generated by a \textbf{random policy}, not
an expert, so no recorded episode reaches the goal on its own. 

\subsection{Micro planners}
\label{sec:appendix-model}

\change{Our micro planners run CEM over two frozen base world models, LeWorldModel (LeWM)~\citep{maes2026lewm} and PLDM~\citep{sobal2025pldm}.}
Both use the LeWM authors' released weights on all four benchmarks. They are based on the same architecture: the encoder is a
ViT-\emph{tiny} (12 layers, hidden width 192, $5.5$M parameters) trained from scratch at $224 \times 224$ input resolution with patch size 14, Its CLS token passes through a projector MLP ($192 \to 2048 \to 192$, with batch normalization (BN) and GELU), giving a single \textbf{192-dimensional} vector (state) per frame. The predictor is an autoregressive transformer ($10.8$M parameters) over these states and has $6$ layers, $16$ self-attention heads, $192$ width, $2048$ feed-forward width, and a positional embedding. The action is conditioned through adaptive layer normalization.

Both planners also use the same CEM hyperparameters: $300$ samples, top-$k$ 30, variance scale $1.0$, action block of $5 frames$. The number of CEM iterations is $30$ on \textsc{PushT} and $10$ on the other benchmarks. Each plan solves for $h$ action blocks and executes them all open-loop before solving again till the budget $T$ is reached. We sweep $h$ only as part of \method's $(h, H)$ search and never change any other solver setting.

\subsection{Hierarchical baselines}
\label{sec:appendix-hier}

The three hierarchical baselines we compare against are the macro planners of~\citet{caselli2026mindthegap}, and all experiments are done using the authors' own code. All three share one architecture: the frozen LeWM \change{with a CEM} micro planner, plus a trained \emph{macro} predictor that maps a latent state and a 32-dimensional latent macro-action to the latent state one macro step later. A macro step is $25$ frames (5 action blocks of 5 frames). The macro planner optimizes a sequence of $h_h$ macro-actions using CEM, optimizing the last predicted latent to $\zgoal$. The first predicted latent, after the first macro step, becomes the sub-goal, which the micro planner pursues by optimizing its local trajectory via its own independent CEM loop. They differ only in how the macro level proposes macro-actions and how often the macro planner replans:
\begin{itemize}
  \item \textbf{Hi-LeWM} runs CEM directly over the latent macro-actions,
    within a box calibrated from 2{,}048 encoded expert action chunks, and
    warm-starts each solve from the previous solution.
  \item \textbf{Hi-LeWM-C} samples instead from a bank of encoded expert
    action sequences: each candidate is a bank exemplar plus a Gaussian
    residual, one candidate is left unperturbed, the residual is refitted to
    the elites with a floor of $10^{-3}$, and the executed candidate is the
    best one found rather than the elite mean. We use the bank the authors
    specified to us, 65{,}536 sequences with residual scale $0.2$. The bank is not
    leave-one-out, unlike \method's graph, so CEM may draw from the very
    demonstration it is evaluated on.
  \item \textbf{Hi-LeWM-C (staged)} uses the same bank solver but solves
    \emph{once}, at $t = 0$. Its $h_h = 2$ predicted latents become fixed
    stage targets, each held for $\lceil d / 2 \rceil$ steps.
\end{itemize}
Hi-LeWM and Hi-LeWM-C replan the macro level every $5$ frames for all benchmarks, apart for \textsc{PushT} where it is replan every $2$ frames at $d = 75$, following the authors' reported setting.

\paragraph{Hyperparameters.} We use the authors' best reported solver budgets, which scale with $d$, in Table.~\ref{tab:hier-cem}.

\begin{table}[h]
  \caption{CEM hyperparameters of the hierarchical baselines, variance scale $1.0$ at
  both levels. They come from the authors' released matrix, which covers
  \textsc{PushT} and \textsc{Cube}; we apply the same rows on
  \textsc{TwoRoom} and \textsc{Reacher}.}
  \label{tab:hier-cem}
  \centering
  \begin{tabular}{lcc}
    \toprule
    $d$ & macro: samples $\times$ iterations, top-$k$ & micro: samples $\times$ iterations, top-$k$ \\
    \midrule
    25 & $900 \times 20$, $10$  & $300 \times 30$, $150$ \\
    50 & $1500 \times 40$, $10$ & $900 \times 30$, $150$ \\
    75 & $1200 \times 60$, $10$ & $1200 \times 30$, $150$ \\
    \bottomrule
  \end{tabular}

\end{table}

These baselines therefore do \emph{not} share the micro solver settings of Section~\ref{sec:appendix-model}; only the checkpoint is shared. Their micro planner plans $2$ action blocks ($10$ frames), executes the first block and replans, with $30$ iterations on every benchmark and more samples as $d$ grows. For macro horizon $h_h$, following the authors, which sweeps $h_h \in \{1, 2, 3\}$ on \textsc{PushT} and \textsc{Cube}, we sweep $h_h \in \{1, 2\}$ and report the performance that is better of the two for each benchmark and budget. The staged arm is fixed at $h_h = 2$ because it needs two stage targets. 

\paragraph{Training on \textsc{TwoRoom} and \textsc{Reacher}.} The authors release trained macro models for \textsc{PushT} and \textsc{Cube}, which we use. For \textsc{TwoRoom} and \textsc{Reacher} we trained our own using the authors' model, loss and data code. Their released training entry point is not functional, so we wrote the missing training pipeline. We verified that pipeline by rebuilding the \textsc{Cube} model from the authors' dumped configuration: its 425 tensors and its trainable set match the released checkpoint exactly. We also normalize pixels to ImageNet statistics, which their encoder path omits and the frozen LeWM encoder
expects. Hyperparameters for the training are read off directly from their paper:
\begin{itemize}
  \item Trained parts: only the macro planner, $12.5$\,M of $30.5$\,M
    parameters; the micro LeWM encoder, predictor, action encoder and projectors
    stay frozen.
  \item Loss: mean squared error between the predicted and actual
    latent of the next waypoint.
  \item Data: windows of $18$ rows at frame-skip $5$ (a 3-row history
    plus a span of up to 15), with 5 waypoints per window sampled at sorted
    random positions.
  \item Optimization: AdamW, learning rate $5 \times 10^{-5}$, weight
    decay $10^{-3}$, global batch $128$, bf16, gradient clipping at $1.0$, and
    a linear warm-up over the first $1\%$ of steps followed by cosine decay to
    zero, for $15$ epochs.
\end{itemize}

Following \textsc{PushT} and \textsc{Cube}, we evaluate the epoch-$15$ checkpoints. Training cost is in
Table.\ref{tab:offline}.

\subsection{\method's hyperparameters}
\label{sec:appendix-hsweep}

Unlike hierarchical baselines, where the macro planner CEM hyperparameters are tuned together with the micro planner CEM hyperparameters on top of macro-specific hyperparameters, \textit{e.g.} interval in which macro planner replans, \method{} has only 2 hyperparameters \change{in the experiments of Sections~\ref{sec:results-main} and \ref{sec:results-analysis}}: the waypoint budget $H$, in frames, and the micro planner's plan length $h$, in action blocks. Every other micro planner CEM hyperparameter follows what the micro planner was shipped with (Section~\ref{sec:appendix-model}) and they stay fixed. $H$ is quantized in steps of the action block $a = 5$, and we fixed the macro planner interval to $2h a$, one micro replan per sub-goal, throughout \change{these experiments}. \change{The protocol of Section~\ref{sec:results-hjepa} adds the number of action blocks executed per plan (Appendix~\ref{app:hjepa-setup}).} We also tried tuning the number of edges $k$ in the graph but got similar robust performance so we fixed $k=4$ for efficiency in all experiments.

On \textsc{PushT} we sweep $h \in \{1, 3, 5\}$ against $H \in [1, 25]$, $[10, 40]$ and $[20, 40]$ respectively, totaling $471$ evaluation runs on 3 different seeds $\{7,8,9\}$. On the other three benchmarks we sweep $h \in \{1, \dots, 5\}$ against $H \in [ha - 5,\, ha + 5]$, clamped at $1$, so $H \le 30$. As success rate is latching, for each $(h, H)$ cell, we run once at max budget $T=4d$ for all benchmarks and the success rate at any budget $<= T$ can be read from it. As each cell is cheap and doesn't require training, the sweep across all benchmarks took under a day on 8 NVIDIA A100 80GB GPUs for each micro planner.

To encourage reproducibility without rerunning the hyperparameter sweep, we provide the cells found, along with their budget $(h,H,T)$, for each \method point in the frontier shown in Figure~\ref{fig:pareto}: 

For \methodlewm:
\begin{itemize}
\item \textsc{PushT}: at $d{=}25$, $(3, 40, 30)$, $(5, 40, 50)$, $(3, 40, 60)$,
$(5, 34, 100)$; at $d{=}50$, $(3, 29, 60)$, $(3, 25, 90)$,
$(3, 20, 120)$, $(5, 27, 200)$; at $d{=}75$, $(3, 30, 90)$,
$(3, 24, 150)$, $(3, 24, 180)$, $(5, 28, 300)$.

\item \textsc{Cube}: at $d{=}25$, $(1, 10, 30)$, $(1, 8, 50)$, $(1, 8, 60)$,
$(1, 6, 100)$; at $d{=}50$, $(1, 10, 50)$, $(1, 10, 60)$, $(1, 8, 100)$,
$(1, 7, 130)$, $(1, 7, 200)$; at $d{=}75$, $(1, 7, T)$ for
$T \in \{90, 150, 200, 300\}$.

\item \textsc{TwoRoom}: at $d{=}25$, $(1, 8, T)$ for $T \in \{30, 50, 60\}$;
at $d{=}50$, $(1, 7, T)$ for $T \in \{50, 60, 100\}$; at $d{=}75$,
$(1, 8, 90)$, $(1, 7, 150)$.

\item \textsc{Reacher}: a single point per distance, $(1, 10, 30)$, $(1, 10, 50)$
and $(1, 10, 90)$ at $d{=}25, 50, 75$, each already at $100\%$.
\end{itemize}

For \methodpldm: 
\begin{itemize}
\item \textsc{PushT}: at $\Delta{=}25$, $(3, 15, T)$ for $T \in \{60, 90\}$; at
$\Delta{=}50$, $(3, 19, T)$ for $T \in \{90, 120, 180\}$; at $\Delta{=}75$,
$(3, 17, 150)$, $(3, 13, 180)$, $(3, 13, 300)$.

\item \textsc{Cube}: at $\Delta{=}25$, $(1, 9, 30)$, $(1, 8, 50)$, $(1, 8, 100)$;
at $\Delta{=}50$, $(1, 9, 50)$, $(1, 9, 60)$, $(1, 6, 100)$, $(1, 7, 130)$,
$(1, 5, 200)$; at $\Delta{=}75$, $(1, 7, 90)$, $(1, 7, 150)$, $(1, 5, 200)$.

\item \textsc{TwoRoom}: at $\Delta{=}25$, $(1, 9, 30)$, $(1, 8, 50)$; at
$\Delta{=}50$, $(1, 8, T)$ for $T \in \{50, 60\}$; at $\Delta{=}75$,
$(1, 8, 90)$, $(1, 9, 150)$.

\item \textsc{Reacher}: $(1, 10, 30)$ and $(1, 10, 50)$ at $\Delta{=}25$, then a
single point, $(1, 10, 50)$ and $(1, 10, 90)$, at $\Delta{=}50$ and $75$.
\end{itemize}

\section{\change{H-JEPA Protocol}}
\label{app:hjepa}

\change{All code, model checkpoints, test tasks, and datasets are obtained as provided by H-JEPA's release, \url{https://github.com/kevinghst/H-JEPA} (commit 97054ec).}

\subsection{\change{Experimental setup}}
\label{app:hjepa-setup}

\change{
\keypoint{Benchmarks.} We use H-JEPA's four benchmarks~\citep{zhang2026hjepa} with its released $50$ test tasks each and its simulators and success tests (Table~\ref{tab:hjepa-bench}). \textsc{Cube} tasks are $d=20$ frames ahead in held-out expert episodes, in windows centred on the grasp; \textsc{PushT} goals are $d=75$ frames ahead in episodes from a separate validation set; \textsc{AntMaze} goals are three maze cells from the start, which the expert policy reaches in $d=61$--$115$ frames; \textsc{FourRoom} goals are $d=75$ frames ahead along an expert path that crosses two rooms with one active distractor, and one of its doors is closed in each episode. \method's graph is built on training episodes only.
}

\begin{table}[h]
  \caption{\change{Total number of episodes in \method's graph and episode budget for each benchmark of the H-JEPA protocol.}}
  \label{tab:hjepa-bench}
  \centering
  \begin{tabular}{lrrr}
    \toprule
    benchmark & no. of episodes/demonstrations & $d$ frames & budget $T$ \\
    \midrule
    \textsc{Cube}     & $9{,}800$ & $20$ & $30$ \\
    \textsc{PushT}    & $18{,}685$ & $75$ & $90$ \\
    \textsc{AntMaze}  & $25{,}000$ & $61$--$115$ & $350$ \\
    \textsc{FourRoom} & $2{,}494$  & $75$ & $210$ \\
    \bottomrule
  \end{tabular}
\end{table}

\change{
\keypoint{Planner and baselines.} All methods use H-JEPA's released checkpoints for seeds $42$, $43$ and $44$ and its gradient-descent planner ($800$ samples optimised with AdamW, scoring the last two predicted latents against the goal), with action blocks of $5$ frames. The baselines run at their released configurations: the flat LeWM shrinks its horizon toward the episode deadline, the hierarchies use their per-level configurations, and all execute two action blocks ($10$ frames) per plan.
}

\change{
\keypoint{\method.} Before every solve, \method re-enters the graph from the current observation (Eq.~\ref{eq:entry}) and replaces the planner's goal with the selected sub-goal. As this sub-goal lies at most $H$ frames ahead, we set $h$ to a fixed value rather than decaying it over the episode; all other planner settings follow H-JEPA's. In this protocol \method has three hyperparameters: $h$, $H$, and the number of action blocks executed per plan. We sweep the number of executed blocks in $\{1,2\}$ and $(h,H)$ on a grid within $h \in \{1,2,3\}$ and $H \in [5, 20]$ on \textsc{Cube}, $H \in [20, 75]$ on \textsc{PushT}, $H \in [5, 30]$ on \textsc{AntMaze} and \textsc{FourRoom}.
The sweep is done on $50$ tuning tasks disjoint from the test set, generated with H-JEPA's own task generators (\textsc{Cube}, \textsc{AntMaze}, \textsc{FourRoom}) or, as no disjoint validation episodes exist on \textsc{PushT}, drawn from its training split with the query episode removed from the graph. 
Executing one block per plan is selected on all benchmarks, with $(h, H)=(1, 12)$ on \textsc{Cube}, $(3, 34)$ on \textsc{PushT}, $(3, 9)$ on \textsc{AntMaze} and $(1, 8)$ on \textsc{FourRoom}. 
}

\end{document}

%% file: math.tex
\usepackage{amsmath,amsfonts,bm}

\def\eqref#1{equation~\ref{#1}}

\def\1{\bm{1}}

\DeclareMathAlphabet{\mathsfit}{\encodingdefault}{\sfdefault}{m}{sl}
\SetMathAlphabet{\mathsfit}{bold}{\encodingdefault}{\sfdefault}{bx}{n}

\DeclareMathOperator*{\argmin}{arg\,min}

%% file: arxiv_main.bbl
\begin{thebibliography}{39}
\providecommand{\natexlab}[1]{#1}
\providecommand{\url}[1]{\texttt{#1}}
\expandafter\ifx\csname urlstyle\endcsname\relax
  \providecommand{\doi}[1]{doi: #1}\else
  \providecommand{\doi}{doi: \begingroup \urlstyle{rm}\Url}\fi

\bibitem[Assran et~al.(2023)Assran, Duval, Misra, Bojanowski, Vincent, Rabbat, LeCun, and Ballas]{assran2023ijepa}
Mahmoud Assran, Quentin Duval, Ishan Misra, Piotr Bojanowski, Pascal Vincent, Michael Rabbat, Yann LeCun, and Nicolas Ballas.
\newblock Self-supervised learning from images with a joint-embedding predictive architecture.
\newblock In \emph{IEEE / CVF Computer Vision and Pattern Recognition Conference (CVPR)}, 2023.
\newblock \doi{10.1109/CVPR52729.2023.01499}.

\bibitem[Baek et~al.(2025)Baek, Park, Park, Oh, and Kim]{baek2025gas}
Seungho Baek, Taegeon Park, Jongchan Park, Seungjun Oh, and Yusung Kim.
\newblock Graph-assisted stitching for offline hierarchical reinforcement learning.
\newblock In \emph{International Conference on Machine Learning (ICML)}, 2025.

\bibitem[Bardes et~al.(2024)Bardes, Garrido, Ponce, Chen, Rabbat, LeCun, Assran, and Ballas]{bardes2024vjepa}
Adrien Bardes, Quentin Garrido, Jean Ponce, Xinlei Chen, Michael Rabbat, Yann LeCun, Mido Assran, and Nicolas Ballas.
\newblock Revisiting feature prediction for learning visual representations from video.
\newblock \emph{Transactions on Machine Learning Research (TMLR)}, 2024.

\bibitem[Black et~al.(2024)Black, Nakamoto, Atreya, Walke, Finn, Kumar, and Levine]{black2024susie}
Kevin Black, Mitsuhiko Nakamoto, Pranav Atreya, Homer Walke, Chelsea Finn, Aviral Kumar, and Sergey Levine.
\newblock Zero-shot robotic manipulation with pretrained image-editing diffusion models.
\newblock In \emph{International Conference on Learning Representations (ICLR)}, 2024.

\bibitem[Caselli et~al.(2026)Caselli, Massafra, Punzo, Lo~Sardo, Pantelidis, and Bhethanabhotla]{caselli2026mindthegap}
Niccol{\`o} Caselli, Francesco Massafra, Samuele Punzo, Salvatore Lo~Sardo, Ippokratis Pantelidis, and Sathya~Kamesh Bhethanabhotla.
\newblock Mind the gap: Promises and pitfalls of hierarchical planning in {LeWorldModel}.
\newblock \emph{arXiv preprint arXiv:2607.12547}, 2026.

\bibitem[Chi et~al.(2024)Chi, Xu, Feng, Cousineau, Du, Burchfiel, Tedrake, and Song]{chi2023diffusionpolicy}
Cheng Chi, Zhenjia Xu, Siyuan Feng, Eric Cousineau, Yilun Du, Benjamin Burchfiel, Russ Tedrake, and Shuran Song.
\newblock Diffusion policy: Visuomotor policy learning via action diffusion.
\newblock \emph{International Journal of Robotics Research}, 2024.
\newblock \doi{10.1177/02783649241273668}.

\bibitem[Dijkstra(1959)]{dijkstra1959}
Edsger~W. Dijkstra.
\newblock A note on two problems in connexion with graphs.
\newblock \emph{Numerische Mathematik}, 1959.
\newblock \doi{10.1007/BF01386390}.

\bibitem[Du et~al.(2023)Du, Yang, Dai, Dai, Nachum, Tenenbaum, Schuurmans, and Abbeel]{du2023unipi}
Yilun Du, Sherry Yang, Bo~Dai, Hanjun Dai, Ofir Nachum, Josh Tenenbaum, Dale Schuurmans, and Pieter Abbeel.
\newblock Learning universal policies via text-guided video generation.
\newblock In \emph{Advances in Neural Information Processing Systems (NeurIPS)}, 2023.

\bibitem[Du et~al.(2024)Du, Yang, Florence, Xia, Wahid, Ichter, Sermanet, Yu, Abbeel, Tenenbaum, Kaelbling, Zeng, and Tompson]{du2024vlp}
Yilun Du, Sherry Yang, Pete Florence, Fei Xia, Ayzaan Wahid, Brian Ichter, Pierre Sermanet, Tianhe Yu, Pieter Abbeel, Joshua~B. Tenenbaum, Leslie~Pack Kaelbling, Andy Zeng, and Jonathan Tompson.
\newblock Video language planning.
\newblock In \emph{International Conference on Learning Representations (ICLR)}, 2024.

\bibitem[Emmons et~al.(2020)Emmons, Jain, Laskin, Kurutach, Abbeel, and Pathak]{emmons2020sgm}
Scott Emmons, Ajay Jain, Misha Laskin, Thanard Kurutach, Pieter Abbeel, and Deepak Pathak.
\newblock Sparse graphical memory for robust planning.
\newblock In \emph{Advances in Neural Information Processing Systems (NeurIPS)}, 2020.

\bibitem[Eysenbach et~al.(2019)Eysenbach, Salakhutdinov, and Levine]{eysenbach2019sorb}
Ben Eysenbach, Russ~R Salakhutdinov, and Sergey Levine.
\newblock Search on the replay buffer: Bridging planning and reinforcement learning.
\newblock In \emph{Advances in Neural Information Processing Systems (NeurIPS)}, 2019.

\bibitem[Florence et~al.(2021)Florence, Lynch, Zeng, Ramirez, Wahid, Downs, Wong, Lee, Mordatch, and Tompson]{florence2021ibc}
Pete Florence, Corey Lynch, Andy Zeng, Oscar~A Ramirez, Ayzaan Wahid, Laura Downs, Adrian Wong, Johnny Lee, Igor Mordatch, and Jonathan Tompson.
\newblock Implicit behavioral cloning.
\newblock In \emph{CoRL}, 2021.

\bibitem[Huang et~al.(2019)Huang, Liu, and Su]{huang2019landmarks}
Zhiao Huang, Fangchen Liu, and Hao Su.
\newblock Mapping state space using landmarks for universal goal reaching.
\newblock In \emph{Advances in Neural Information Processing Systems (NeurIPS)}, 2019.

\bibitem[Jayaraman et~al.(2019)Jayaraman, Ebert, Efros, and Levine]{jayaraman2019tap}
Dinesh Jayaraman, Frederik Ebert, Alexei~A. Efros, and Sergey Levine.
\newblock Time-agnostic prediction: Predicting predictable video frames.
\newblock In \emph{International Conference on Learning Representations (ICLR)}, 2019.

\bibitem[Kim et~al.(2021)Kim, Seo, and Shin]{kim2021higl}
Junsu Kim, Younggyo Seo, and Jinwoo Shin.
\newblock Landmark-guided subgoal generation in hierarchical reinforcement learning.
\newblock In \emph{Advances in Neural Information Processing Systems (NeurIPS)}, 2021.

\bibitem[Kurutach et~al.(2018)Kurutach, Tamar, Yang, Russell, and Abbeel]{kurutach2018causalinfogan}
Thanard Kurutach, Aviv Tamar, Ge~Yang, Stuart Russell, and Pieter Abbeel.
\newblock Learning plannable representations with causal {InfoGAN}.
\newblock In \emph{Advances in Neural Information Processing Systems (NeurIPS)}, 2018.

\bibitem[LeCun(2022)]{lecun2022path}
Yann LeCun.
\newblock A path towards autonomous machine intelligence.
\newblock OpenReview, 2022.

\bibitem[Lee et~al.(2022)Lee, Kim, Jang, and Kim]{lee2022dhrl}
Seungjae Lee, Jigang Kim, Inkyu Jang, and H.~Jin Kim.
\newblock {DHRL}: A graph-based approach for long-horizon and sparse hierarchical reinforcement learning.
\newblock In \emph{Advances in Neural Information Processing Systems (NeurIPS)}, 2022.
\newblock \doi{10.52202/068431-0993}.

\bibitem[Liu et~al.(2020)Liu, Kurutach, Tung, Abbeel, and Tamar]{liu2020htm}
Kara Liu, Thanard Kurutach, Christine Tung, Pieter Abbeel, and Aviv Tamar.
\newblock Hallucinative topological memory for zero-shot visual planning.
\newblock In \emph{International Conference on Machine Learning (ICML)}, 2020.

\bibitem[Maes et~al.(2026)Maes, Le~Lidec, Scieur, LeCun, and Balestriero]{maes2026lewm}
Lucas Maes, Quentin Le~Lidec, Damien Scieur, Yann LeCun, and Randall Balestriero.
\newblock {LeWorldModel}: Stable end-to-end joint-embedding predictive architecture from pixels.
\newblock \emph{arXiv preprint arXiv:2603.19312}, 2026.

\bibitem[Masip et~al.(2026)Masip, Swinnen, Hu, Detry, and Tuytelaars]{masip2026ffjepa}
Sergi Masip, Jonathan Swinnen, Yutong Hu, Renaud Detry, and Tinne Tuytelaars.
\newblock {FF-JEPA}: Long-horizon planning in world models with latent planners.
\newblock \emph{arXiv preprint arXiv:2606.09311}, 2026.

\bibitem[Motamed et~al.(2026)Motamed, Culp, Swersky, Jaini, and Geirhos]{motamed2026physicsiq}
Saman Motamed, Laura Culp, Kevin Swersky, Priyank Jaini, and Robert Geirhos.
\newblock Do generative video models understand physical principles?
\newblock In \emph{IEEE/CVF Winter Conference on Applications of Computer Vision (WACV)}, 2026.

\bibitem[Nair et~al.(2018)Nair, Pong, Dalal, Bahl, Lin, and Levine]{nair2018rig}
Ashvin Nair, Vitchyr Pong, Murtaza Dalal, Shikhar Bahl, Steven Lin, and Sergey Levine.
\newblock Visual reinforcement learning with imagined goals.
\newblock In \emph{Advances in Neural Information Processing Systems (NeurIPS)}, 2018.

\bibitem[Nair \& Finn(2020)Nair and Finn]{nair2020hvf}
Suraj Nair and Chelsea Finn.
\newblock Hierarchical foresight: Self-supervised learning of long-horizon tasks via visual subgoal generation.
\newblock In \emph{International Conference on Learning Representations (ICLR)}, 2020.

\bibitem[Nasiriany et~al.(2019)Nasiriany, Pong, Lin, and Levine]{nasiriany2019leap}
Soroush Nasiriany, Vitchyr~H. Pong, Steven Lin, and Sergey Levine.
\newblock Planning with goal-conditioned policies.
\newblock In \emph{Advances in Neural Information Processing Systems (NeurIPS)}, 2019.

\bibitem[Opryshko et~al.(2026)Opryshko, Quan, Voelcker, Du, and Gilitschenski]{opryshko2026ttgs}
Evgenii Opryshko, Junwei Quan, Claas Voelcker, Yilun Du, and Igor Gilitschenski.
\newblock Test-time graph search for goal-conditioned reinforcement learning.
\newblock In \emph{International Conference on Machine Learning (ICML)}, 2026.

\bibitem[Pari et~al.(2022)Pari, Shafiullah, Arunachalam, and Pinto]{pari2022vinn}
Jyothish Pari, Nur Muhammad~(Mahi) Shafiullah, Sridhar~Pandian Arunachalam, and Lerrel Pinto.
\newblock The surprising effectiveness of representation learning for visual imitation.
\newblock In \emph{Proceedings of Robotics: Science and Systems (RSS)}, 2022.
\newblock \doi{10.15607/RSS.2022.XVIII.010}.

\bibitem[Park et~al.(2025)Park, Frans, Eysenbach, and Levine]{park2025ogbench}
Seohong Park, Kevin Frans, Benjamin Eysenbach, and Sergey Levine.
\newblock {OGBench}: Benchmarking offline goal-conditioned {RL}.
\newblock In \emph{International Conference on Learning Representations (ICLR)}, 2025.

\bibitem[Pertsch et~al.(2020)Pertsch, Rybkin, Ebert, Zhou, Jayaraman, Finn, and Levine]{pertsch2020lvd}
Karl Pertsch, Oleh Rybkin, Frederik Ebert, Shenghao Zhou, Dinesh Jayaraman, Chelsea Finn, and Sergey Levine.
\newblock Long-horizon visual planning with goal-conditioned hierarchical predictors.
\newblock In \emph{Advances in Neural Information Processing Systems (NeurIPS)}, 2020.

\bibitem[Rubinstein(1999)]{rubinstein1999cem}
Reuven Rubinstein.
\newblock The cross-entropy method for combinatorial and continuous optimization.
\newblock \emph{Methodology And Computing In Applied Probability}, 1\penalty0 (2):\penalty0 127–190, September 1999.
\newblock ISSN 1573-7713.
\newblock \doi{10.1023/a:1010091220143}.
\newblock URL \url{http://dx.doi.org/10.1023/A:1010091220143}.

\bibitem[Savinov et~al.(2018)Savinov, Dosovitskiy, and Koltun]{savinov2018sptm}
Nikolay Savinov, Alexey Dosovitskiy, and Vladlen Koltun.
\newblock Semi-parametric topological memory for navigation.
\newblock In \emph{International Conference on Learning Representations (ICLR)}, 2018.

\bibitem[Shah et~al.(2021)Shah, Eysenbach, Kahn, Rhinehart, and Levine]{shah2021ving}
Dhruv Shah, Benjamin Eysenbach, Gregory Kahn, Nicholas Rhinehart, and Sergey Levine.
\newblock {ViNG}: Learning open-world navigation with visual goals.
\newblock In \emph{IEEE International Conference on Robotics and Automation (ICRA)}, 2021.
\newblock \doi{10.1109/ICRA48506.2021.9561936}.

\bibitem[Sobal et~al.(2025)Sobal, Zhang, Cho, Balestriero, Rudner, and LeCun]{sobal2025pldm}
Uladzislau Sobal, Wancong Zhang, Kyunghyun Cho, Randall Balestriero, Tim G.~J. Rudner, and Yann LeCun.
\newblock Learning from reward-free offline data: A case for planning with latent dynamics models.
\newblock In \emph{Advances in Neural Information Processing Systems (NeurIPS)}, 2025.
\newblock \doi{10.52202/085713-1465}.

\bibitem[Tassa et~al.(2018)Tassa, Doron, Muldal, Erez, Li, de~Las~Casas, Budden, Abdolmaleki, Merel, Lefrancq, Lillicrap, and Riedmiller]{tassa2018deepmind}
Yuval Tassa, Yotam Doron, Alistair Muldal, Tom Erez, Yazhe Li, Diego de~Las~Casas, David Budden, Abbas Abdolmaleki, Josh Merel, Andrew Lefrancq, Timothy Lillicrap, and Martin Riedmiller.
\newblock {DeepMind} control suite.
\newblock \emph{arXiv preprint arXiv:1801.00690}, 2018.

\bibitem[Todorov et~al.(2012)Todorov, Erez, and Tassa]{todorov2012mujoco}
Emanuel Todorov, Tom Erez, and Yuval Tassa.
\newblock {MuJoCo}: A physics engine for model-based control.
\newblock In \emph{IEEE/RSJ International Conference on Intelligent Robots and Systems (IROS)}, 2012.

\bibitem[Yang et~al.(2020)Yang, Zhang, Morcos, Pineau, Abbeel, and Calandra]{yang2020plan2vec}
Ge~Yang, Amy Zhang, Ari Morcos, Joelle Pineau, Pieter Abbeel, and Roberto Calandra.
\newblock {Plan2Vec}: Unsupervised representation learning by latent plans.
\newblock In \emph{Learning for Dynamics \& Control Conference (L4DC)}, 2020.

\bibitem[Zhang et~al.(2021)Zhang, Yang, and Stadie]{zhang2021l3p}
Lunjun Zhang, Ge~Yang, and Bradly~C Stadie.
\newblock World model as a graph: Learning latent landmarks for planning.
\newblock In \emph{International Conference on Machine Learning (ICML)}, 2021.

\change{
\bibitem[Zhang et~al.(2026a)Zhang, Terver, Zholus, Chitnis, Sutaria, Assran, Balestriero, Bar, Bardes, LeCun, et~al.]{zhang2026hierarchical}
Wancong Zhang, Basile Terver, Artem Zholus, Soham Chitnis, Harsh Sutaria, Mido Assran, Randall Balestriero, Amir Bar, Adrien Bardes, Yann LeCun, et~al.
\newblock Hierarchical planning with latent world models.
\newblock \emph{arXiv preprint arXiv:2604.03208}, 2026.
}

\bibitem[Zhou et~al.(2025)Zhou, Pan, LeCun, and Pinto]{zhou2024dinowm}
Gaoyue Zhou, Hengkai Pan, Yann LeCun, and Lerrel Pinto.
\newblock {DINO-WM}: World models on pre-trained visual features enable zero-shot planning.
\newblock In \emph{International Conference on Machine Learning (ICML)}, 2025.

\change{
\bibitem[Zhang et~al.(2026b)Zhang, Terver, Rabbat, LeCun, Balestriero]{zhang2026hjepa}
Wancong Zhang, Basile Terver, Michael Rabbat, Yann LeCun, and Randall Balestriero.
\newblock H-JEPA: End-to-End Learning of Hierarchical World Models for Visual Planning.
\newblock \emph{arXiv preprint arXiv:2610.06805}, 2026.
}

\end{thebibliography}
